\documentclass[letterpaper]{article} 

\pdfoutput=1

\PassOptionsToPackage{numbers, compress, sort}{natbib}

\usepackage{times}  
\usepackage{helvet}  
\usepackage{courier}  
\usepackage{url}  
\usepackage{graphicx}  
\usepackage{natbib}
\usepackage[utf8]{inputenc} 
\usepackage[T1]{fontenc}    
\usepackage{hyperref}       
\usepackage{url}            
\usepackage{booktabs}       
\usepackage{amsfonts}       
\usepackage{amsmath}  
\usepackage{amsthm}
\usepackage{amssymb}  
\usepackage{nicefrac}       
\usepackage{microtype}      
\usepackage{bbm}
\usepackage{tikz}
\usetikzlibrary{arrows}
\usetikzlibrary{positioning}
\tikzset{
  treenode/.style = {align=center, inner sep=0pt, text centered,
    font=\sffamily},
  arn_n/.style = {treenode, circle, black, font=\sffamily\bfseries, draw=black,
    fill=white, text width=1.5em},
  arn_r/.style = {treenode, circle, black, font=\sffamily\bfseries, draw=black,
    fill=white, text width=1.0em},
  arn_x/.style = {treenode, rectangle, draw=black,
    minimum width=0.5em, minimum height=0.5em}
}
\usepackage{thmtools,thm-restate}
\usepackage{parskip}
\usepackage{authblk}
\usepackage[linesnumbered,ruled,noend]{algorithm2e}
\usepackage{algorithmic}
\usepackage{mathtools}

\usepackage{floatrow}
\newfloatcommand{capbtabbox}{table}[][\FBwidth]

\usepackage{basic}

\newtheorem{lemma}{Lemma}

\newtheorem{definition}{Definition}

\newcommand{\argmin}[2]{\textrm{argmin}_{#1}~#2}

\newcommand{\E}[2]{\mathbb{E}_{#1}\left[#2\right]}
\newcommand{\Ehat}[1]{\hat{\mathbb{E}}\left[#1\right]}
\newcommand{\Var}[2]{\textrm{Var}_{#1}\left[#2\right]}

\newcommand{\ind}[1]{\mathbbm{1}\left\{#1\right\}}

\newcommand{\supp}{\text{supp}}
\newcommand\independent{\protect\mathpalette{\protect\independenT}{\perp}}
\def\independenT#1#2{\mathrel{\rlap{$#1#2$}\mkern2mu{#1#2}}}
\newcommand{\sgn}[0]{\textrm{sgn}}
\newcommand{\lf}[0]{\lambda}
\newcommand{\lfbar}[0]{\bar{\lambda}}
\newcommand{\phat}[0]{\hat{p}}
\newcommand{\bmin}[0]{b_{\text{min}}}
\newcommand{\NN}[0]{\text{NN}}

\newcommand{\D}[0]{\mathcal{D}}

\newcommand{\X}[0]{\mathcal{X}}
\newcommand{\Y}[0]{\mathcal{Y}}
\newcommand{\p}[0]{\mathcal{P}}

\newcommand{\C}[0]{\mathcal{C}}

\newcommand{\Z}[0]{\mathcal{Z}}

\newcommand{\sysname}{\textsc{Liger}}

\ifdefined\ShowNotes
  \newcommand{\colornote}[3]{{\color{#1}\bf{#2 #3}\normalfont}}
\else
  \newcommand{\colornote}[3]{}
\fi

\definecolor{darkred}{rgb}{0.7,0.1,0.1}
\definecolor{darkgreen}{rgb}{0.1,0.5,0.1}
\definecolor{cyan}{rgb}{0.7,0.0,0.7}
\definecolor{dblue}{rgb}{0.2,0.2,0.8}
\definecolor{maroon}{rgb}{0.76,.13,.28}
\definecolor{burntorange}{rgb}{0.81,.33,0}
\definecolor{royalpurple}{rgb}{0.47,.31,0.66}

\newcommand{\spam}{\textbf{Spam}}
\newcommand{\spouse}{\textbf{Spouse}}
\newcommand{\weather}{\textbf{Weather}}

\newcommand{\commercial}{\textbf{Commercial}}
\newcommand{\tennis}{\textbf{Tennis}}
\newcommand{\basketball}{\textbf{Basketball}}

\ifdefined\ShowNotes
  \newcommand{\num}[1]{{\color{red}\bf{#1}\normalfont}}
\else
  \newcommand{\num}[1]{#1}
\fi

\usepackage{etoolbox}

\makeatletter
\newcommand{\changeoperator}[1]{%
  \csletcs{#1@saved}{#1@}%
  \csdef{#1@}{\changed@operator{#1}}%
}
\newcommand{\changed@operator}[1]{%
  \mathop{%
    \mathchoice{\textstyle\csuse{#1@saved}}
               {\csuse{#1@saved}}
               {\csuse{#1@saved}}
               {\csuse{#1@saved}}%
  }%
}
\makeatother

\let\oldnl\nl
\newcommand{\nonl}{\renewcommand{\nl}{\let\nl\oldnl}}

\newif\ifsinglecolumn
\singlecolumntrue

\ifsinglecolumn
\title{Shoring Up the Foundations: \\ Fusing Model Embeddings and Weak Supervision}
\else
\title{Shoring Up the Foundations: Fusing Model Embeddings and Weak Supervision}
\fi

\author[1]{Mayee~F.~Chen$^*$}
\author[1]{Daniel~Y.~Fu\thanks{Equal Contribution. A preliminary version of the results in this paper can be found at \url{https://arxiv.org/abs/2006.15168}.}}
\author[2]{Dyah~Adila}
\author[1]{Michael~Zhang}
\author[2]{Frederic~Sala}
\author[1]{Kayvon~Fatahalian}
\author[1]{Christopher~R\'e}

\affil[1]{Department of Computer Science, Stanford University}
\affil[2]{Department of Computer Science, University of Wisconsin-Madison}
\affil[1]{\footnotesize{\texttt{\{mfchen, danfu, mzhang, kayvonf, chrismre\}@cs.stanford.edu}}}
\affil[2]{\footnotesize{\texttt{\{adila, fredsala\}@cs.wisc.edu}}}

\begin{document}

\maketitle

\begin{abstract}

Foundation models offer an exciting new paradigm for constructing models with out-of-the-box embeddings and a few labeled examples.
However, it is not clear how to best apply foundation models without labeled data.
A potential approach is to fuse foundation models with weak supervision frameworks, which use weak label sources---pre-trained models, heuristics, crowd-workers---to construct pseudolabels.
The challenge is building a combination that best exploits the signal available in both foundation models and weak sources.
We propose \sysname, a combination that uses foundation model embeddings to improve two crucial elements of existing weak supervision techniques.
First, we produce finer estimates of weak source quality by partitioning the embedding space and learning per-part source accuracies.
Second, we improve source coverage by extending source votes in embedding space.
Despite the black-box nature of foundation models,
we prove results characterizing how our approach improves performance and show that lift scales with the smoothness of label distributions in embedding space.
On six benchmark NLP and video tasks,
\sysname\ outperforms vanilla weak supervision by \num{14.1} points, weakly-supervised kNN and adapters by \num{11.8} points, and kNN and adapters supervised by traditional hand labels by \num{7.2} points.

\end{abstract}


\section{Introduction}
\label{sec:intro}

\begin{figure*}[t]
    \includegraphics[width=6.5in]{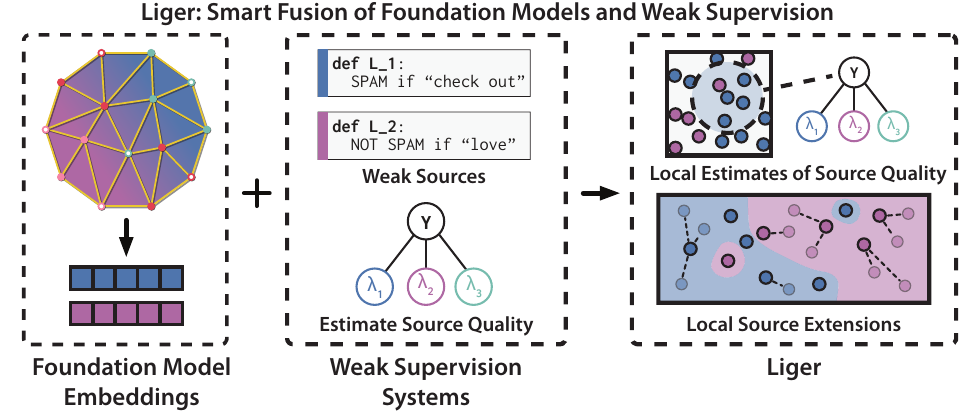}
    \centering
    \caption{
      \sysname\ fuses embeddings from foundation models (left) with weak supervision (middle) by exploiting local smoothness of the representations (right).
      \sysname\ uses the embeddings to a) produce more accurate local estimates of weak source quality (right, top), and b) to locally extend weak sources, improving their coverage (right, bottom).
    }
    \label{fig:banner}
  \end{figure*}

Foundation models---large pretrained models such as GPT-3, BERT, CLIP, and DALL-E \citep{brown2020language, devlin-2019-bert, radford2021learning, dalle}---offer powerful representations that can be used in a broad array of settings~\citep{Bommasani21FM}. 
These models have achieved state-of-the-art performance on many tasks.
However, it remains unclear how to best apply foundation models in situations where users lack access to any labeled data but do have some weak signals.
These are the cases where another class of techniques---weak supervision \citep{Ratner18, fu2020fast}---shines.

The broad success of foundation models (FMs) suggests that fusing them with weak supervision may offer substantial benefits.
Intuitively, the signals present in both can be used to replace large amounts of hand-labeled data in supervised learning. These signals are complementary.
Foundation models are trained on huge amounts of data and thus offer powerful general-purpose embeddings.
Weak supervision frameworks rely on multiple weak sources of signal that can be synthesized into pseudolabels for downstream training.
These weak sources typically express specialized domain expertise. 
The fusion may enable each component to be improved: FM embeddings can be used without labeled data, while weak sources may be extended to be more general-purpose.

Our goal is to combine these complementary signals to address two challenges in existing approaches to weak supervision.
The first challenge is performing fine-grained estimation of source quality.
Current weak supervision approaches typically coarsely model source quality by assuming error distributions are uniform over unlabeled points~\citep{ratner2019training, fu2020fast}, but source quality may vary across points in actuality.
The second challenge is producing votes on points where sources abstain. 
Weak sources often abstain, so that current approaches suffer from low coverage and have many points lacking any signal. 
We seek to exploit the powerful embeddings from FMs---and the geometry induced by them---to address these challenges.

We propose \sysname, a new weak supervision approach based on the notion of \emph{local} quality of weak sources in the FM embedding space (named after a well-known fusion of powerful animals).
We introduce an efficient algorithm that partitions the embedding space and learns per-part local source accuracies. 
\sysname\ also extends weak sources into nearby regions of the embedding space that they previously abstained on, improving coverage.
Despite the fact that FMs are typically black-box, our localized approach exploits a simple measurable notion of their signal: the smoothness of the label distribution in the embedding space.
When the distribution of label values does not vary significantly over an embedding region, local source accuracies can be estimated well, and local source extensions maintain their accuracy.
We introduce generalization error bounds that individually characterize the impact of partitioning and extending.
These error bounds scale in the embedding smoothness and involve a bias-variance tradeoff in 
the number of partitions and the radii that specify extensions, suggesting that careful incorporation of the FM's signal into our approach is necessary.

We evaluate \sysname\ on six benchmark NLP and video weak supervision tasks, fusing weak sources with GPT-3 embeddings~\citep{brown2020language, neelakantan2022text} for the NLP tasks, and with image embeddings from CLIP~\citep{radford2021learning} for the video tasks.
We compare \sysname\ against using FMs or weak supervision on their own, as well as baseline techniques for fusing them together.
First, \sysname\ outperforms two strong baselines for traditional supervision of FMs, kNN and adapters~\citep{houlsby2019parameter}, by \num{7.2} points, and outperforms traditional weak supervision by \num{14.1} points.
Next, \sysname\ outperforms kNN or adapter-based fusions of weak supervision and FMs by \num{11.8} points.
We find that lift scales with embedding smoothness---confirming our theoretical findings.
We measure the smoothness of CLIP embeddings against BiT-M~\citep{kolesnikov2020big}, ResNet-101 embeddings pretrained on ImageNet~\citep{ILSVRC15}, and raw pixels on a video task.
We find that CLIP embeddings are smoothest and result in the best performance.
Similarly, we find that using the right prompt for GPT-3 has a strong effect on smoothness and performance on a relation extraction task.

In summary, we contribute:
\begin{itemize}[leftmargin=*]
    \item \sysname, a new approach for fusing foundation models with weak supervision by exploiting local smoothness of labels and weak sources in embedding space.
    \item Finite-sample generalization error bounds of our algorithm that scale in this smoothness.
    \item Evaluation of \sysname\ on six benchmark NLP and video weak supervision tasks, where \sysname\ outperforms simple fusions of foundation models and weak supervison, as well as either on its own.
\end{itemize}


\section{Background}
\label{sec:background}

We describe the problem setting for weak supervision (Section~\ref{sec:setup}). We introduce two general challenges in weak supervision that our approach using foundation model embeddings can mitigate.
We then propose a model and explain its two stages---source quality estimation and pseudolabel inference (Section~\ref{sec:lm}). 
We provide a brief background on the estimation technique from~\cite{fu2020fast}, on top of which we build our approach.

\subsection{Problem Setup} \label{sec:setup}
Our goal is to predict label $y \in \Y = \{-1,+1\}$ from datapoints $x \in \X$. If we had access to pairs $(x, y)$, we could train a supervised model. However, we do not have access to any samples of $y$; instead, we observe $m$ \emph{weak sources} $\bm{\lf} = \{\lf_1, \dots, \lf_m\}$, each
voting or abstaining on each point $x$ via a probabilistic \textit{labeling function}
$\lf_j: \X \rightarrow \Y \cup \{0\}$ for all $j \in [m]$. 
We refer to $\lf_j(x) = 0$ as an abstain, which occurs when a source is uncertain or not applicable on a point. 

We also have access to FM embeddings. These embeddings are the outputs of a mapping $f: \X \rightarrow \Z$ from input space to an embedding space $\Z$ equipped with metric $\rho: \Z \times \Z \rightarrow \mathbb{R}^+$. 
This mapping is fixed and obtained from an off-the-shelf model.
Overall, we have an unlabeled dataset $\D = \{x_i\}_{i = 1}^n$ of $n$ i.i.d. points, as well as access to $m$ weak sources and the embedding map $f$. 

Given an input $x$ and $\bm{\lf}(x)$, we aim to learn a \emph{label model} that predicts $y$ by estimating $\hat{\Pr}(y | \bm{\lf}, x)$ (we drop the $x$ in $\bm{\lf}(x)$ when obvious). The goal of the label model is to combine sources based on their individual accuracies (i.e. $\lf_i$'s rate of agreement with $y$) by upweighting high-quality sources and downweighting low-quality ones. The resulting pseudolabels given by $\hat{\Pr}(y | \bm{\lf}, x)$ can be used to train a downstream supervised \emph{end model} or used just directly as predictions. The latter case is often ideal, since users need not train an additional model. We focus on this setting.

\textbf{Two Challenges and Opportunities.} 
Next, we describe two challenges common to weak supervision techniques. Fusing weak supervision with FM embeddings presents opportunities to mitigate these challenges.
\begin{itemize}[itemsep=0.5pt,topsep=0pt,leftmargin=*]
\item {\bf Coarse Accuracy Modeling.} The most common assumption in weak supervision is to model $\hat{\Pr}(y | \bm{\lf}, x)$ as $\hat{\Pr}(y | \bm{\lf})$. That is, conditioned on the weak sources, the true label is viewed as independent of the features, so only one set of accuracies is learned over the data. Removing this assumption is desirable, since the feature space may have information about the task not captured fully by weak sources. However, naively attempting to model per-point accuracies leads to noisy estimation.
\item {\bf Low Coverage.} Weak sources frequently abstain, leading to low coverage---a situation where much of the dataset has no votes. A simple mitigation is to extend votes from nearby non-abstaining points, but this is risky if the notion of distance is not well-aligned with the label distribution. 
\end{itemize}

An intuitive way to tackle these two challenges is to operate \emph{locally}. 
Suppose the source votes and the true label satisfy some level of smoothness such that within some local region of the feature space, they have a low probability of changing values. 
We can then model accuracies specific to such local regions and can extend source votes to points they abstain on within the regions. 
However, raw image and text features may lack signal and not offer sufficient smoothness to permit operating locally. 
By acting on the embedding space, the desired smoothness property is improved (see Figure~\ref{smoothness}). We can thus obtain finer-grained accuracy estimation and improved coverage by using FM embeddings to model local accuracies and extend locally.

Next, we make these notions concrete by presenting the explicit model for $\Pr(y, \bm{\lf} | x)$.

\subsection{Label Model} \label{sec:lm}
We model $\Pr(y, \bm{\lf} | x)$ as a probabilistic graphical model. 
Our use of this model has two steps. First, in training, we must estimate the accuracy parameters of $\Pr(y, \bm{\lf} | x)$ without access to $y$. Then, at inference, we compute $\hat{\Pr}(y | \bm{\lf}, x)$.

Let the graphical model be based on $G = (V, E)$, where $V = y \cup \bm{\lf}$ and $E$ consists of edges from $y$ to each $\lf_j$ (see Figure~\ref{fig:banner} middle). 
For simplicity, we assume there are no dependencies between the weak sources, although the dependencies can be learned~\citep{varma2019learning} and handled by our choice of base estimator from~\citep{fu2020fast}. Therefore, our approach can be extended to that case as well.
We model the data distribution as
\ifsinglecolumn
\begin{align}
\Pr(y, \bm{\lf} | x) = \frac{1}{Z} \exp &\Big(\underbrace{\theta_y(x) y}_{\text{Class Balance}} + \sum_{i = 1}^m \underbrace{\theta_i(x) \lf_i y}_{\text{Source Accuracy}} + \sum_{i = 1}^m \underbrace{\theta_{i, 0}(x) \ind{\lf_i = 0}}_{\text{Abstain Rate}} \Big) 
\label{eq:pgm} 
\end{align}
\else
\begin{align}
\Pr(y, \bm{\lf} | x) = \frac{1}{Z} \exp &\Big(\underbrace{\theta_y(x) y}_{\text{Class Balance}} + \sum_{i = 1}^m \underbrace{\theta_i(x) \lf_i y}_{\text{Source Accuracy}} \nonumber \\
&\qquad   + \sum_{i = 1}^m \underbrace{\theta_{i, 0}(x) \ind{\lf_i = 0}}_{\text{Abstain Rate}} \Big) 
\label{eq:pgm} 
\end{align}
\fi
with partition function $Z$ and a set of canonical parameters per $x$, $\Theta(x) = \{ \theta_y(x), \theta_i(x), \theta_{i, 0}(x) \; \forall i \in [m]\}$. 
An important property above is that $\lf_i \independent \lf_j | y, x\; \forall i, j \in [m]$.

The model concretely portrays the two challenges in weak supervision. 
First, canonical parameters $\Theta(x)$ that are a function of the input can capture varying accuracy across the data. 
This is less strict than prior formulations that model the marginal $\Pr(y, \bm{\lf})$ with one set of canonical parameters without considering input data.
However, estimating $\Theta(x)$ is challenging; parametric approaches require certain assumptions on the function $\Theta$ as well as the distribution of $x$ in order to recover the ground truth labels, but these assumptions (e.g., Gaussian $x$) are often not realistic. Standard nonparametric approaches have a high computational complexity and rely on smoothness of the input space $\X$.
Second, when $\lf_i(x) = 0$, the weak source provides no information on $x$ at inference and is thus typically ignored on that point in previous approaches. 
This is reflected in the graphical model by Lemma~\ref{lemma:abstain} in Appendix~\ref{sec:supp_pgm}, by which $\Pr(y | \lf_i = 0, \bm{\lf} \backslash \lf_i, x) = \Pr(y | \bm{\lf} \backslash \lf_i, x)$. 
In fact, the weak sources provide no direct signal on $x$  when $\bm{\lf}(x) = \vec{0}$. 

\textbf{Pseudolabel Inference.} To perform inference, we compute $\hat{\Pr}(y | \bm{\lf}, x)$ for some $x\in\X$.  This is done via Bayes' rule and the conditional independence of weak sources: $\Pr(y | \bm{\lf}, x) = \prod_{i = 1}^m \Pr(\lf_i | y, x) \Pr(y | x) / \Pr(\bm{\lf} | x)$. The latent parameter of interest in this decomposition is $\Pr(\lf_i | y, x)$, which corresponds to the accuracy of $\lf_i$.

\textbf{Source Parameter Estimation.} Previous approaches have considered how to estimate $\Pr(\lf_i | y)$ in a model of $\Pr(\lf, y)$ via the \emph{triplet method}~\citep{fu2020fast}, using conditional independence properties.
For our setting, \eqref{eq:pgm} tells us that $\lf_i y \independent \lf_j y | \lf_i \wedge \lf_j \neq 0, x$ for any $i \neq j$ (Lemma~\ref{lemma:triplet_independence} in Appendix~\ref{sec:supp_pgm}). 
As a result, $\E{}{\lf_i y | \lf_i \neq 0, x} \times \E{}{\lf_j y | \lf_j \neq 0, x} = \E{}{\lf_i \lf_j y^2 | \lf_i \wedge \lf_j \neq 0, x} = \E{}{\lf_i \lf_j | \lf_i \wedge \lf_j \neq 0, x}$, which consists of observable variables. 
Define $a_i(x) = \E{}{\lf_i y | \lf_i \neq 0, x}$ as the \emph{accuracy} of $\lf_i$ on $x$. 
If we introduce a third $\lf_k$, we can generate a system of equations over $a_i(x), a_j(x), a_k(x)$ in terms of the conditional expected products of pairs of $\lf_i, \lf_j, \lf_k$. 
As a result, 
\ifsinglecolumn
\begin{align}
|a_i(x) | :=  \sqrt{\bigg| \frac{\E{}{\lf_i \lf_j | \lf_i \wedge \lf_j \neq 0, x} \E{}{\lf_i \lf_k | \lf_i \wedge \lf_k \neq 0, x}}{\E{}{\lf_j \lf_k | \lf_j \wedge \lf_k \neq 0, x}}\bigg|}, \label{eq:triplet}
\end{align}
\else
\begin{align}
&|a_i(x) | := \label{eq:triplet} \\
&\sqrt{\bigg| \frac{\E{}{\lf_i \lf_j | \lf_i \wedge \lf_j \neq 0, x} \E{}{\lf_i \lf_k | \lf_i \wedge \lf_k \neq 0, x}}{\E{}{\lf_j \lf_k | \lf_j \wedge \lf_k \neq 0, x}}\bigg|}, \nonumber
\end{align}
\fi
and likewise for $\hat{a}_j(x), \hat{a}_k(x)$.
More details are in Appendix~\ref{sec:supp_triplet}.
\eqref{eq:pgm} allows us to write $\Pr(\lf_i | y, x) = \frac{1 + \sgn(\lf_i y) a_i(x)}{2} \times \Pr(\lf_i \neq 0 | x)$ (Lemmas~\ref{lemma:abstain} and~\ref{lemma:symmetry}), so the desired probability estimate is just a linear transformation of $a_i(x)$ scaled by $\lf_i$'s coverage.


\section{Fusion Algorithm}
\label{sec:method}

We are ready to present \sysname, our approach to fusing foundation model embeddings and weak supervision.
We explain the two components: first, how to compute conditional estimates of the label model parameters over local regions of the partitioned embedding space for finer-grained accuracy estimation;
second, how to extend weak sources via a kNN-like augmentation in the embedding space, improving their coverage and hence the signal available at inference.
The full approach is shown in Algorithm~\ref{alg:main}. 

\begin{algorithm}[t]
	\caption{\sysname}
	\begin{algorithmic}
		\STATE \textbf{Input:}
		Dataset $\D = \{x_i\}_{i = 1}^n$, weak sources $\bm{\lf}$,  embedding mapping $f$ and metric $\rho$, threshold radii $r_1, \dots r_m$, partition $\C$ and class balances $\Pr(y | C_j)$ for $j \in [s]$.
		\STATE \textbf{Returns:} Label model $\hat{\Pr}(y | \bm{\lfbar}, x)$.
		\FOR{$\lf_i \in \bm{\lf}$}
			\STATE Construct extended source $\lfbar_i$ using $r_i, f, \rho$ as in~\eqref{eq:extended}.
		\ENDFOR
		\FOR{$C_j \in \C$}
			\FOR {$\lfbar_i \in \bm{\lfbar}$}
				\STATE Compute accuracy $\hat{a}_i(C_j)$ using Algorithm~\ref{alg:triplet} on $\lfbar_i$ over $C_j$, and compute coverage $\hat{\Pr}(\lfbar_i \neq 0 | C_j)$ on $\D$. 
				\STATE Set $\hat{\Pr}(\lfbar_i | y, C_j)$ equal to $\frac{1 + \sgn(\lfbar_i y) \hat{a}_i(C_j)}{2} \hat{\Pr}(\lfbar_i \neq 0 | C_j) $ for $\lfbar_i \in \{-1, 1\}$, $\hat{\Pr}(\lfbar_i = 0 | C_j)$ otherwise.
			\ENDFOR
			\STATE Compute $\hat{\Pr}(\bm{\lfbar} | C_j)$ on $\D$.
		\ENDFOR
		\RETURN For test point $x \in \X$, compute $\hat{\Pr}(y | \bm{\lfbar}, x) = \hat{\Pr}(y | \bm{\lfbar}, C(x)) = \frac{\prod_{i = 1}^m \hat{\Pr}(\lfbar_i | y, C(x)) \Pr(y | C(x))}{\hat{\Pr}(\bm{\lfbar} | C(x))}$.
	\end{algorithmic}
	\label{alg:main}
\end{algorithm}

\paragraph{Local Parameter Estimation}

Our first task is to compute the label model's local parameters. Based on~\eqref{eq:triplet}, the quantities to estimate are of the form $\E{}{\lf_i \lf_j | \lf_i \wedge \lf_j \neq 0, x}$, $\Pr(\lf_i \neq 0 | x)$, $\Pr(\bm{\lf} | x)$, $\Pr(y | x)$. 
These conditional statistics can be estimated using nonparametric approaches such as the Nadaraya-Watson estimator, but they require $\mathcal{O}(n)$ computations per point at inference. 

Instead of estimating parameters per point, we partition the embedding space and compute \emph{per-part} statistics. Intuitively, this choice exploits smoothness.
If label distributions are smooth, i.e., they do not vary greatly within a local region, it is sufficient to estimate per-point statistics using a part given that parts are not too large. Controlling the size of the partition is thus important in determining how well we can approximate per-point statistics.

Concretely, partition $\Z$ into $s$ subsets $\C = \{C_1, \dots, C_s\}$ of equal size $n' = \frac{n}{s}$ (we use K-means clustering with $K=s$ in practice).
Denote $C(x)$ as the subset $f(x)$ belongs to. 
Instead of estimating statistics and performing inference conditioned on $x$, we condition on $C(x)$, producing $s$ sets of parameters overall. 
We estimate $\E{}{\lf_i \lf_j | \lf_i \wedge \lf_j \neq 0, C(x)}$, yielding a local accuracy estimate $\hat{a}_i(C(x))$ formalized in Algorithm~\ref{alg:triplet}, as well as $\Pr(\lf_i \neq 0 | C(x))$, $\Pr(\bm{\lf} | C(x))$, $\Pr(y | C(x))$. Then, we use $\hat{\Pr}(y | \bm{\lf}, x) = \hat{\Pr}(y | \bm{\lf}, C(x))$ as our label model prediction on $x$.
These estimates are done over the subsets; for instance, $\Pr(\bm{\lf} | C(x)) \approx \frac{1}{n'}\sum_{x' \in C(x)} \ind{\bm{\lf}(x') \texttt{=} \bm{\lf}}$.
We assume that class balance on subsets, $\Pr(y | C(x))$, are known.
There are also several techniques that can be used to estimate these~\citep{ratner2019training}, or they can be treated as hyperparameters.

\paragraph{Weak Source Extension}
Next, we improve the model of $\hat{\Pr}(y | \bm{\lf}, x)$ by increasing source coverage. 
Let $\lfbar_i$ be an extended labeling function with corresponding threshold radius $r_i > 0$ for $i \in [m]$.
The extension works as follows.
For any $x$, let $\NN(x) = \argmin{x' \in \D: \lf_i(x) \neq 0} \rho(f(x), f(x'))$ be the nearest neighbor of $x$ in embedding space from $\D$ such that $\lf_i$ has coverage on it. 
$\lfbar_i$ uses nearest neighbors to weakly label points within $r_i$ of $\lf_i$'s support on $\D$. Formally,
\begin{align}
\lfbar_i(x) := \begin{cases}
\lf_i(x) & \lf_i(x) \neq 0 \\
\lf_i(\NN(x)) & \rho(f(x), f(\NN(x))) \le r_i \\ 
0 & \text{o.w.}
\end{cases}. \label{eq:extended}
\end{align}

We can view $\lfbar_i$ as an augmentation on $\lf_i$ using $\D$ and $f$.
We thus perform parameter estimation and inference using $\bm{\lfbar}$ instead of $\bm{\lf}$, namely learning $\Pr(y | \bm{\lfbar}, C(x))$. 

There are two advantages to using extended sources.
First, extended sources improve sampling error, since expressions like $\E{}{\lf_i \lf_j | \lf_i \wedge \lf_j \neq 0, x}$ are estimated over more data in $\D$.
Second, $\lfbar_i$ provides signal at inference on points that $\lf_i$ previously abstains on. However, the quality of this signal greatly depends on $r_i$. If $\lf_i$ is overextended and the embedding space is not sufficiently smooth, points far away from $\lf_i$'s support may receive incorrect extended source votes, suggesting that careful choice of $r_i$ is needed.

Our approach combines the two components discussed---partitioning the embedding space and extending sources---to output predictions $\hat{\Pr}(y | \bm{\lfbar}, C(x))$ as in Algorithm~\ref{alg:main}. Note that our approach builds on the algorithm from~\cite{fu2020fast}, but partitioning and extending can also be done on top of other weak supervision algorithms that model things differently.


\section{Theoretical Analysis}
\label{sec:theory}
Now we turn to analyzing Algorithm~\ref{alg:main}. Our goal is to understand how performance depends on the key parameters: fineness of the partition $\C$, radii $r_i$ of the extensions used to improve coverage, and smoothness of the embedding space.

We begin with a result on the generalization error of the label model $\hat{\Pr}(y | \bm{\lf}, x)$, which relies on the number of partitions $s$ to control the granularity of the estimated parameters (Theorem~\ref{thm:gen_err}). Then, we discuss the improvement from using $\bm{\lfbar}$ instead of $\bm{\lf}$. We first bound the local accuracy of an extended source in a region it previously abstains (Lemma~\ref{lemma:extended_acc}), and then we show that as long as this local accuracy is better than random, we can further reduce the generalization error (Theorem~\ref{thm:lift}). The former result presents a bias-variance tradeoff depending on $s$, while the latter has a tradeoff dependent on the threshold radius $r_i$. In both cases, $s$ and $r_i$ must be carefully set based on the signal in the FM embeddings, namely the smoothness of label distributions in the FM embedding space, in order to optimize performance. We provide proofs in Appendix~\ref{sec:supp_proofs}, synthetic experiments supporting our findings in Appendix~\ref{sec:supp_exp_synthetics}, and smoothness measurements on real data in Section~\ref{sec:exp-smoothness} and Appendix~\ref{sec:supp_smooth}.

Define the generalization error of the label model using weak sources $\bm{\lf}$ as the expected cross-entropy loss, $L(\bm{\lf}) = \mathbb{E}_{\D, x, y, \bm{\lf}}[-\log \hat{\Pr}(y | \bm{\lf}, x)]$.

\subsection{Label Model Generalization Error} \label{subsec:gen_err}
We bound the generalization error $L(\bm{\lf})$ of the label model using the unextended, initial weak sources. The key quantity in this analysis is embedding smoothness: 
\begin{definition}[Lipschitzness]
The distributions $\Pr(y | x)$ and $\Pr(\lf_i | y, x)$
are \emph{Lipschitz-smooth} on the metric space $(\Z, \rho)$ with constants $K_y, K_{\lf}, K_{\lf, 0} > 0$ if for all $i \in [m]$,
\ifsinglecolumn
\begin{align*}
&|\Pr(y = 1 | x) - \Pr(y = 1 | x')| \le K_y \rho(f(x), f(x')), \\
&|\Pr(\lf_i = 1 | y, \lf_i \neq 0, x) - \Pr(\lf_i = 1 | y, \lf_i \neq 0, x')| \le K_{\lf}\rho(f(x), f(x')), \\
&|\Pr(\lf_i \neq 0 | x) - \Pr(\lf_i \neq 0 | x')| \le K_{\lf, 0} \rho(f(x), f(x')),
\end{align*}
\else
\begin{align*}
&|\Pr(y = 1 | x) - \Pr(y = 1 | x')| \le K_y \rho(f(x), f(x')), \\
&|\Pr(\lf_i = 1 | y, \lf_i \neq 0, x) - \Pr(\lf_i = 1 | y, \lf_i \neq 0, x')| \\
& \quad \le K_{\lf}\rho(f(x), f(x')), \\
&|\Pr(\lf_i \neq 0 | x) - \Pr(\lf_i \neq 0 | x')| \le K_{\lf, 0} \rho(f(x), f(x')),
\end{align*}
\fi
We refer to these three properties as label, source, and coverage Lipschitzness, respectively.

\label{assumption:lipschitzness}
\end{definition}
In words, if the constants $K_y, K_{\lf}, K_{\lf, 0}$ are small, the class balance of $y$ and the way each source votes (or doesn't) do not vary significantly over a local region of the embedding space.

We define some additional quantities. Set $\alpha = \max_i \E{x}{\frac{1}{p_{ij}} \; \big| \; p_{ij} \neq 0}$, where $p_{ij} = \Pr(\lf_i \neq 0 | f(x) \in C_j)$ is the coverage of $\lf_i$ on $C_j$, to be the largest average inverse source coverage over the subsets. $\alpha$ corresponds to how often sources abstain. Assume that $a_i(C_j) > 0$ for all $\lf_i$ and $C_j$, meaning that the average source accuracy on a subset is better than random. Then, define $a_{\max} = \max_{i,j} a_i(C_j)$, and $b_{\min} = \min\limits_{i, j, k} \{\E{}{\lf_i \lf_k | \lf_i \wedge \lf_k \neq 0, C_j}, \Ehat{\lf_i \lf_k | \lf_i \wedge \lf_k \neq 0, C_j} \}$ as the minimum rate of agreement between sources over subsets, where $\hat{\mathbb{E}}$ denotes the empirical estimate on $\D$. Define $d_{C_j}=\max_{f(x), f(x') \in C_j} \rho(f(x), f(x'))$ as the diameter of $C_j$ and $d_{\C}=\E{x}{d_{C(x)}}$ as its average.

\begin{restatable}[]{theorem}{generr}
Suppose that data $x, y, \bm{\lf}$ follows the model in~\eqref{eq:pgm} and $\Pr(y | x)$ and $\Pr(\lf_i | y, x)$ for each $\lf_i$ are Lipschitz-smooth. The generalization error of the label model $\hat{\Pr}(y | \bm{\lf}, x)$ in Algorithm~\ref{alg:main} when $r_i = 0 \; \forall i$ can be decomposed into $L(\bm{\lf}) \texttt{=} \text{Bias} + \text{Variance} + \text{Irreducible Error} + o(1/n)$, where
\ifsinglecolumn
\begin{align*}
&\text{Bias} \le 2 d_{\C}(K_y + mK_{\lf} + mK_{\lf, 0}), \\
&\text{Variance} \le \frac{ms}{ n} \bigg(\frac{3 \alpha (1 - b_{\min}^2)}{8b_{\min}^2 (1 - a_{\max}^2)} \Big(\frac{1}{b_{\min}^4} + \frac{2}{b_{\min}^2} \Big) + 1 \bigg), \\
&\text{Irreducible Error} = H(y | \bm{\lf}, x),
\end{align*}
\else
\begin{align*}
&\text{Bias} \le 2 d_{\C}(K_y + mK_{\lf} + mK_{\lf, 0}), \\
&\text{Variance} \le \frac{ms}{n} \bigg(\frac{3\alpha (1 - b_{\min}^2)}{8b_{\min}^2 (1 - a_{\max}^2)} \Big(\frac{1}{b_{\min}^4} + \frac{2}{b_{\min}^2} \Big) + 1 \bigg), \\
&\text{Irreducible Error} = H(y | \bm{\lf}, x),
\end{align*}
\fi
where $H(y | \bm{\lf}, x)$ denotes conditional entropy.
\label{thm:gen_err}
\end{restatable}

We discuss each term of this bound.
\begin{itemize}[itemsep=0.5pt,topsep=0pt,leftmargin=*]
\item The bias comes from the partition $\C$, since conditional statistics on $C(x)$ are not equivalent to those on $x$. When the embedding space is smooth with small $K_y, K_{\lf}, K_{\lf, 0}$, the bias is low. Note that making the subset diameter $d_C \rightarrow 0$ makes the bias go to zero.
\item The variance comes from sampling error in Algorithm~\ref{alg:triplet} and $\hat{\Pr}(\lf_i \neq 0 | C_j)$. This quantity scales in $\mathcal{O}(s\alpha /n)$ and also depends on accuracy and agreement among weak sources.
\item The irreducible error depends on quality of $\bm{\lf}$. If knowledge of $\bm{\lf}$ significantly reduces uncertainty in $y$, i.e., the sources contain lots of signal, this quantity is low. On the other hand, $H(y | \bm{\lf}, x)$ is maximized when $\bm{\lf} \independent y | x$, i.e. there is no signal about $y$ in $\bm{\lf}$.
\end{itemize}

Our result reveals a bias-variance tradeoff dependent on the number of parts $s$. As $s$ increases, subset diameter $d_{\C}$ tends to decrease, resulting in lower bias because the subset parameters estimated will be closer in true value to those conditional on $x$. The variance increases in $s$ because there are fewer points per subset for estimation. 
The $s = 1$ case, which incurs a large bias, is algorithmically equivalent to the approach in~\cite{fu2020fast}. Such approaches thus suffer from model misspecification in our setting---and likely in most practical cases---as they assume uniform quality per source.

\subsection{Improvement from Extensions}\label{subsec:lift}
Suppose that $x, y, \bm{\lfbar}$ follows~\eqref{eq:pgm}.
When we use $\bm{\lfbar}$ rather than $\bm{\lf}$ (i.e. $r_i \neq 0$), there are several changes to the decomposition in Theorem~\ref{thm:gen_err}:
\begin{itemize}[itemsep=0.5pt,topsep=0pt,leftmargin=*]
\item The bias is now bounded by $2d_{\C}K_y + 2m(d_{\C} + 2\max_i r_i) (K_{\lf} + K_{\lf, 0})$ (see Lemma~\ref{lemma:ext_bias} in Appendix~\ref{sec:supp_proofs}). We must consider when $\NN(x)$ is not in $C(x)$, essentially resulting in a wider subset diameter.
\item The variance is still $\mathcal{O}(1/n)$, but multiplicative factors change. For instance, $\alpha$ decreases due to improved coverage, thus decreasing the variance. 
\item The irreducible error is now $H(y | \bm{\lfbar}, x)$. 
\end{itemize}

We analyze $H(y | \bm{\lfbar}, x)$ in this section.
$\lfbar_i$ provides more signal than $\lf_i$ at inference on points where $\lf_i(x) = 0$, but the signal about $y$'s value may be incorrect. Extending $\lf_i$ using too large of $r_i$ could yield incorrect source votes, resulting in lower accuracy of the extended weak source.

We first present a result on how $r_i$ controls the extended source's accuracy. Define $a_i = \E{}{\lf_i y | \lf_i \neq 0}$ as the average accuracy of $\lf_i$, and $\bar{a}_i(r_i) = \E{}{\lfbar_i y | \lfbar_i \neq 0, \lf_i = 0}$ as $\lfbar_i$'s average accuracy on the extended region.
We also need a notion of smoothness of $y$ between the original support and the extended region. We define a local notion of \emph{probabilistic Lipschitzness} (PL), originally introduced in~\cite{urner2013probabilistic}. 

\begin{definition}[Probabilistic Lipschitzness]
Define $\p_{\lf_i} = \Pr_{x, y}(\cdot | \lf_i \neq 0)$ to be the distribution of $(x, y)$ over the support of $\lf_i$, and let $\p_{\lf_i, x}$ be its marginal distribution on $x$. Then $\p_{\lf_i}$ is $M$-\emph{probabilistically Lipschitz} for an increasing function $M: \mathbb{R}^+ \rightarrow [0, 1]$ if for any $r > 0$,
\ifsinglecolumn
\begin{align*}
\Pr_{x, y \sim \p_{\lf_i}}(\exists (x', y') &\in \X \backslash \supp(\p_{\lf_i, x}) \times \Y: \rho(f(x), f(x')) \le r, y' \neq y) \le M(r).
\end{align*}
\else
\begin{align*}
\Pr_{x, y \sim \p_{\lf_i}}(\exists (x', y') &\in \X \backslash \supp(\p_{\lf_i, x}) \times \Y: \\
& \rho(f(x), f(x')) \le r, y' \neq y) \le M(r).
\end{align*}
\fi
We refer to this property as local label PL. 
\label{def:pl}
\end{definition}

In words, the probability that there is a point outside of the support of $\lf_i$ but within $r$ of $(x, y) \sim \p_{\lf_i}$ with a different label from $y$ is bounded by an increasing function of $r$. We also define $\beta_i = \mathbb{E}[\lf_i y | \lf_i \neq 0, \exists (x', y'): \lf_i(x') = 0, \rho(f(x), f(x')) \le r_i, y' = y]$ as $\lf_i$'s accuracy over a region close to where $\lf_i$ is extended and $y$ changes value.

With this definition, we show that:
\begin{restatable}[]{lemma}{extendedacc}
Suppose $\p_{\lf_i}$ is $M$-probabilistically Lipschitz. The average accuracy of $\lfbar_i$ on the extended region is at least
$\bar{a}_i(r_i) \ge a_i - (1 + \beta_i) M(r_i)$.
\label{lemma:extended_acc}
\end{restatable}

Our result provides local accuracy guarantees on $\lfbar_i$ as a function of the original $\lf_i$'s accuracy, the probabilistic Lipschitzness of the embedding space, and the $r_i$ the user sets.
Extending a source with higher original accuracy will yield stronger accuracy guarantees in the extended region. 
On the other hand, if $M(r_i)$ is too large due to improper $r_i$ or lack of smoothness, the true label is more likely to change value, and hence accuracy in the extended region worsens. 

Now we can use our result on $\bar{a}_i(r_i)$ to analyze the improvement in irreducible error. We extend just one weak source $\lf_i$ by $r_i$ and keep $\lf_{-i} := \bm{\lf} \backslash \lf_i$ unextended. 
Define $p_i = \Pr(\lfbar_i \neq 0, \lf_i = 0)$ as the proportion of the region where $\lfbar_i$ is extended and $p(\lf_{-i}) = \E{y', \lf_{-i}, \lfbar_i \neq 0, \lf_i = 0}{\Pr(y = y' | \lf_{-i}, x)}$ as the label model's probability of outputting the correct label in the extension region when only using $\lf_{-i}$.

\begin{restatable}[]{theorem}{lift}
Suppose that data follows the model in~\eqref{eq:pgm}. The irreducible error decreases by at least the following amount when using $\lfbar_i$ rather than $\lf_i$ in Algorithm~\ref{alg:main}:
\begin{align*}
H(y | \bm{\lf}, x) - H(y | \bm{\lfbar}, x) \ge 2 p_i (1 - p(\lf_{-i}))^2 \cdot \bar{a}_i(r_i)^2.
\end{align*}
\label{thm:lift}
\end{restatable}

Lift increases with probability mass $p_i$ on the extended region since more of the data is impacted by $\lfbar_i$. Lift is not as significant if $p(\lf_{-i})$ is large because the other weak sources already are providing sufficient signal for $y$. Most importantly, lift scales with how far $\bar{a}_i(r_i)$ is from $0$ (random voting). This highlights a tradeoff in $r_i$: as $r_i$ increases, $p_i$ increases but the lower bound on $\bar{a}_i(r_i)$ from Lemma~\ref{lemma:extended_acc} decreases. This shows that threshold radii must be selected carefully; if the embedding space has strong probabilistic Lipschitzness (i.e. small $M$) or the original weak source has high accuracy, then the source can be extended further while providing lift. However, overextension of the source can yield low local accuracy and thus less lift.

Our results demonstrate that $s$ and $r_i$ control the label model's performance, and setting these terms depends on how smooth label distributions are in the embedding space.


\section{Experiments}
\label{sec:exp}

\begin{table*}[h!]
    \centering
    \begin{minipage}{4.25in}
        \centering
        \scriptsize
        \begin{tabular}{@{}rlcccccccc@{}}
            \toprule
            && \multicolumn{4}{c}{Weak Sources Only}  \\
            \cmidrule(lr){3-6}                                  
            & \textbf{Task} & \textbf{WS-kNN} &  \textbf{WS-Adapter} & \textbf{WS-LM} & \textbf{\sysname}  & \textbf{$\Delta$Coverage}    \\
            \midrule                                                                                                            
            \parbox[t]{0mm}{\multirow{3}{*}{\rotatebox[origin=c]{90}{\textbf{NLP}}}}
            & \spam\        & 72.8           & 92.3  & 83.6             & \textbf{95.0}  & +45.5  \\
            & \weather\     & 62.0           & 86.0  & 78.0             & \textbf{98.0}  & +90.2  \\
            & \spouse\      & 16.9           & 17.1  & 47.0             & \textbf{52.2}  & +12.1  \\
            \midrule                                                                                                                                    
            \parbox[t]{0mm}{\multirow{3}{*}{\rotatebox[origin=c]{90}{\textbf{Video}}}}
            & \basketball\  & 33.3           & 48.9            & 27.9           & \textbf{69.6} & +8.3 \\
            & \commercial\  & 84.7           & 92.8            & 88.4           & \textbf{93.5} & +18.8   \\
            & \tennis\      & 83.0           & \textbf{83.8}   & 82.0           & 83.3          & +32.5 \\
            \bottomrule
        \end{tabular}
    \end{minipage}
    \begin{minipage}{2.4in}
        \scriptsize
        \centering
        \begin{tabular}{@{}cccccccccc@{}}    
            \toprule
            \multicolumn{3}{c}{Dev Labels Available}  \\ 
            \cmidrule(lr){1-3}                                  
            \textbf{kNN} &  \textbf{Adapter} & \textbf{\sysname-Adapter}  \\
            \midrule                                                                                                            
            91.2             & 94.4             & \textbf{95.4}              \\
            92.0             & 90.0             & \textbf{96.8}              \\
            21.6             & 15.7             & \textbf{49.6}  \\
            \midrule                                                                                                                                    
            64.4          & 79.3             & \textbf{79.5} \\
            92.0          & 93.0             & \textbf{93.2}       \\
            73.2          & 83.1             & \textbf{84.0}  \\
            \bottomrule
        \end{tabular}
    \end{minipage}
    
    \caption{Left: \sysname\ performance compared to baselines that only have access to weak labels, as well as the change in coverage from traditional weak supervision.
    Right: \sysname-Adapter performance compared to baselines that have access to dev labels.
    Scores are F1 except for Spam and Weather (accuracy); best score in bold in each setting.}
    \label{table:main_results}
\end{table*}

This section evaluates the following claims about \sysname:
\begin{itemize}[itemsep=0.5pt,topsep=0pt,leftmargin=*]
    \item \textbf{Performance (Section~\ref{sec:exp-performance}):} \sysname\ outperforms vanilla weak supervision, as well as baseline approaches for using foundation models directly, either with traditional weak supervision or hand supervision.
    \item \textbf{Smoothness (Section~\ref{sec:exp-smoothness}):} Lift is correlated with the smoothness of the label distribution in the representation space.
    We measure smoothness and performance of CLIP against three other embedding methods on a video task, and measure three prompting strategies for GPT-3 on a relation extraction task.
    \item \textbf{Ablations (Section~\ref{sec:exp-ablations}):} Both components of \sysname---partitioning the representation space and extending labeling function votes---are important for performance.
\end{itemize}

\paragraph{Datasets}
We evaluate \sysname\ on six benchmark NLP and video tasks used to evaluate previous weak supervision methods~\citep{fu2020fast,zhang2021wrench}.
In NLP, \spam\ identifies spam YouTube comments~\citep{alberto2015tubespam}; \weather\ identifies the sentiment of weather-related tweets~\citep{CrowdflowerWeather}; and \spouse\ identifies spouse relationships in newspaper articles~\citep{corney2016million}.
In video, \commercial\ identifies commercial segments in TV news~\citep{hong2021analysis, fu2019rekall}; \tennis\ identifies rallies in tennis segments; and \basketball\ identifies basketball videos in a subset of ActivityNet~\citep{caba2015activitynet}. Each dataset consists of a large unlabeled training set, a smaller hand-labeled \textit{development set} (train/dev split sizes from 187/50 points to 64,130/9,479 points), and a held-out test set. We use the unlabeled training set to train label models and use the development set for a) training of traditional supervision baselines, and b) hyperparameter tuning of the label models, including $s$ and $r_i$.

\paragraph{Pre-trained embeddings} For the NLP datasets, we use pre-trained GPT-3 \citep{brown2020language} embeddings from OpenAI's Ada model.
For \spam\ and \weather, we simply embed the text directly.
For \spouse, we add a prompt ``Are [person 1] and [person 2] spouses?'' after the end of the sentence.
We discuss further prompting strategies in Section~\ref{sec:exp-smoothness}.
For video datasets, we use image embeddings from CLIP~\citep{radford2021learning} over individual frames of the videos.


\begin{figure*}[t!]
    \centering
    \begin{minipage}{1.15in}
        \scriptsize
        \begin{tabular}{lr}    
            \toprule                                                       
            \textbf{Embedding} & F1-score \\
            \midrule         
            Raw pixel          & 19.3  \\
            RN-101        & 31.1\\
            BiT-M          & 42.5  \\
            \textbf{CLIP}            & \textbf{69.6}  \\
            \bottomrule
        \end{tabular}
    \end{minipage}
    \begin{minipage}{1.85in}
        \centering
        \includegraphics[width=1.85in]{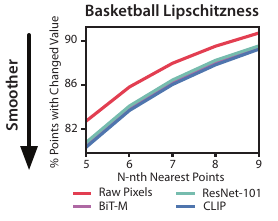}
    \end{minipage}
    \begin{minipage}{1.7in}
        \begin{flushright}
            \scriptsize
            \begin{tabular}{lr}    
                \toprule                                                     
                \textbf{Prompting} & F1-score \\
                \midrule
                No Prompt          & 48.5  \\
                Prompt Beginning          & 50.2  \\
                \textbf{Prompt End}          & \textbf{52.2} \\
                \bottomrule
            \end{tabular}
        \end{flushright}
    \end{minipage}
    \begin{minipage}{1.95in}
        \centering
        \includegraphics[width=1.85in]{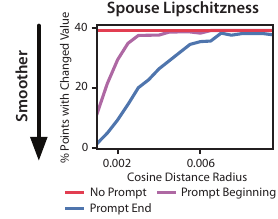}
    \end{minipage}
    \caption{
        Left: \sysname\ performance and smoothness measurements of CLIP, BiT-M, ResNet-101, and raw pixels as embeddings for \basketball.
        Right: \sysname\ performance and smoothness measurements of no prompting, prompting at beginning, and prompting at end in GPT-3 for \spouse.
    }
    \label{smoothness}
\end{figure*}

\subsection{Performance}\label{sec:exp-performance}

We compare \sysname\ against baseline approaches for fusing foundation models with weak supervision, as well as against using either on their own.\footnote{Our code is available at \url{https://github.com/HazyResearch/liger/}.}
We split our evaluation into two parts: methods that only have access to weak sources, and methods that additionally have access to the dev set.

\paragraph{Weak Sources Only}
We compare the performance of \sysname\ against vanilla weak supervision's label model (WS-LM)~\citep{fu2020fast}, as well as two end models, weakly-supervised kNN (WS-kNN), and weakly-supervised adapters (WS-Adapter).
In the latter two methods, we use the predictions from WS-LM to generate pseudolabels for the train set and use the FM embeddings as input data (since we do not access the full FM) to the kNN and adapter approaches.
We consider an adapter that is a linear layer on the FM embeddings. We also provide results on 3-layer MLP adapters in Appendix~\ref{sec:supp_exp}.

Table~\ref{table:main_results} (left) shows the results, as well as statistics on the additive change in coverage ($\%$ of the dataset that sources vote on) between \sysname\ and WS-LM.
\sysname\ outperforms WS-LM and has better coverage (33.2 points on average).
\sysname\ also outperforms both of the baseline approaches for fusing foundation models with weak supervision, WS-kNN and WS-Adapter.

\paragraph{Weak Sources and Dev Labels}
Next, we compare performance against methods that have access to a small hand-labeled dev set.
We compare against two baselines: kNN and Adapter, both trained over the dev set labels.
For our method \sysname-Adapter, we train an adapter over \sysname\ labels on the train set, as well as the dev labels.
In some cases, \sysname\ labels are too noisy to provide good signal on the train set; in this case, our solution automatically downsamples the pseudolabels on the train set.
We also provide the original \sysname\ prediction as input to the adapter.
See Appendix~\ref{sec:supp_details} for the details.

Table~\ref{table:main_results} (right) shows the results.
\sysname-Adapter outperforms Adapter and kNN.
On the datasets where \sysname\ labels are very accurate, we see additional lift from the adapters because we have more points to train on.
When the labels are not very accurate, our downsampling prevents the noisy labels from harming adapter performance.
In one case, learning an adapter over the embeddings is very hard (\spouse).
Here, providing the \sysname\ prediction as input is critical for performance.

\subsection{Embedding Smoothness}\label{sec:exp-smoothness}

We measure how smoothness of the embedding space affects the performance of \sysname.
First, we compare embeddings from CLIP against BiT-M embeddings~\citep{kolesnikov2020big}, a ResNet-101 pretrained on ImageNet~\citep{ILSVRC15}, and raw pixels.
Second, we vary the GPT-3 prompting strategy for \spouse\ and compare against two alternative methods that result in a less smooth representation.
We report label Lipschitzness---the smoothness of embeddings with respect to ground-truth labels---in this section.
See Appendix~\ref{sec:supp_smooth} for additional measures of Lipschitzness.

Figure~\ref{smoothness} (left) shows the performance of CLIP, BiT-M, ResNet-101, and raw pixels as embeddings for \sysname, as well as measures of Lipschitzness for each method (lower is smoother).
CLIP embeddings are smoother than the other methods---which matches their performance when used in \sysname.

\paragraph{Comparing Prompting Strategies}

Next, we examine the impact of prompting strategies for \spouse.
\spouse\ is a relation extraction dataset, where the task is to predict whether two entities in a sentence are married.
Since there may be multiple entities in a sentence, \spouse\ contains multiple duplicate sentences in the dataset, with different labels.
To alleviate this problem, we introduce a prompt ``Are [person 1] and [person 2] spouses?'' after the end of the sentence, where ``[person 1/2]'' are replaced by the names of the first/second entity in the sentence.
We compare this prompting strategy against two others: appending the same prompt to the beginning of the sentence, and leaving the original sentence as-is, without any prompting.

Figure~\ref{smoothness} (right) shows the performance and smoothness of each of these prompting methods.
Adding the prompt to the end of the sentence results in the best performance and smoothest embeddings.
Both methods perform better than leaving the sentence alone (the flat line is a result of multiple sentences with different labels having the same embedding).

\subsection{Ablations}\label{sec:exp-ablations}


\begin{table}[t!]
    \centering
    \ifsinglecolumn
        \normalsize
    \else
        \scriptsize
    \fi
    \begin{tabular}{lcccccc}
    \toprule                                                
    \textbf{Task}  & \textbf{\sysname\ ($s$)} & \textbf{-Part} & \textbf{-Ext ($s$)} & \textbf{-Part, -Ext} \\
    \midrule                                                 
    \spam          & \textbf{95.0} (2)        & 94.0           & 92.0 (7)            & 83.6  \\
    \weather       & \textbf{98.0} (3)        & 96.0           & 94.0 (5)            & 78.0   \\
    \spouse        & \textbf{52.2} (6)        & 50.0           & 49.3 (5)            & 47.0   \\
    \midrule                                                   
    \basketball    & \textbf{69.6} (2)        & \textbf{69.6}  & 21.9 (2)            & 27.9  \\
    \commercial    & \textbf{93.5} (3)        & 92.3           & 91.4 (5)            & 88.4  \\
    \tennis        & \textbf{83.3} (1)        & \textbf{83.3}  & 81.3 (2)            & 82.0  \\
    \bottomrule
    \end{tabular}

    \caption{Ablations of \sysname, removing partitions (-Part), extensions (-Ext), and both. Best $s$ values inside parentheses.}
    \label{table:ablations_k}
\end{table}

We report ablations on each component of \sysname.
Table~\ref{table:ablations_k} removes the partioning component and the extensions component.
Partitioning improves performance on \num{four} tasks, and extensions improves performance on all tasks (\num{13.1} points of lift on average from partitioning, \num{3.8} points from extensions).
Combining both additionally offers the best performance on \num{four} tasks.


\section{Related Work}
\label{sec:related}

We present an abbreviated related work here. See Appendix~\ref{sec:supp_related} for an extended treatment.

Weak supervision frameworks typically model source accuracies to generate weak labels and then fine-tune an end model for generalization~\citep{Ratner18, bach2018snorkel, khetan2017learning, sheng2020gmail, fu2020fast, zhan2019sequentialws, safranchik2020weakly, boecking2019pairwise}.
One framework models the end-to-end process all at once~\citep{cachay2021endtoend}, but requires training the end model at the same time---which is computationally expensive with large foundation models.
Our work removes the fine-tuning step completely.

Our work is similar to transfer learning techniques, which adapt pretrained models for downstream tasks~\citep{kolesnikov2020big,devlin2018bert}.
Foundation models offer new requirements for transfer learning setting: when it is impossible to fine-tune the original models~\citep{Bommasani21FM}.
We build on approaches such as prompting~\citep{lester2021power, brown2020language}, embedding search~\citep{neelakantan2022text}, and adapters~\citep{houlsby2019parameter,alain2016understanding}.


\section{Conclusion}
\label{sec:conc}
We present \sysname, a system for fusing foundation models and weak supervision.
We use embeddings to produce finer-grained estimates of weak source accuracies and improve weak source coverage.
We prove a series of results on how the performance of this approach scales with the smoothness of the embeddings, and demonstrate \sysname\ on six benchmark NLP and video weak supervision datasets.
We hope our work will encourage further work in combining foundation models and weak supervision and in utilizing the signal from foundation models to help with other tasks.

\ifsinglecolumn
\paragraph{Authors' Note}
The first two authors contributed equally. Co-first authors can prioritize their names when adding this paper's reference to their resumes.

\section*{Acknowledgments}
\fi

We thank Fait Poms and Ravi Teja Mullapudi for helpful discussions.
We thank Neel Guha, Megan Leszczynski, Vishnu Sarukkai, and Maya Varma for feedback on early drafts of this paper.
We gratefully acknowledge the support of NIH under No. U54EB020405 (Mobilize), NSF under Nos. CCF1763315 (Beyond Sparsity), CCF1563078 (Volume to Velocity), 1937301 (RTML), and CCF2106707 (Program Synthesis for Weak Supervision); ARL under No. W911NF-21-2-0251 (Interactive Human-AI Teaming); ONR under No. N000141712266 (Unifying Weak Supervision); ONR N00014-20-1-2480: Understanding and Applying Non-Euclidean Geometry in Machine Learning; N000142012275 (NEPTUNE); NXP, Xilinx, LETI-CEA, Intel, IBM, Microsoft, NEC, Toshiba, TSMC, ARM, Hitachi, BASF, Accenture, Ericsson, Qualcomm, Analog Devices, Google Cloud, Salesforce, Total, the HAI-GCP Cloud Credits for Research program,  the Stanford Data Science Initiative (SDSI), Department of Defense (DoD) through the National Defense Science and
Engineering Graduate Fellowship (NDSEG) Program, Wisconsin Alumni Research Foundation (WARF), and members of the Stanford DAWN project: Facebook, Google, and VMWare. The U.S. Government is authorized to reproduce and distribute reprints for Governmental purposes notwithstanding any copyright notation thereon. Any opinions, findings, and conclusions or recommendations expressed in this material are those of the authors and do not necessarily reflect the views, policies, or endorsements, either expressed or implied, of NIH, ONR, or the U.S. Government.

\bibliographystyle{plain}
\bibliography{main}

\newpage 
\appendix
\section*{Appendix}

We present an extended related work (Appendix~\ref{sec:supp_related}), glossary (Appendix~\ref{sec:supp_glossary}), additional algorithmic details (Appendix~\ref{sec:supp_alg}), proofs (Appendix~\ref{sec:supp_proofs}), experimental details (Appendix~\ref{sec:supp_details}), and additional experimental results (Appendix~\ref{sec:supp_exp}).

\section{Extended Related Work}
\label{sec:supp_related}

\textbf{Weak supervision} is a broad set of techniques using weak sources of
signal to supervise models, such as
distant supervision~\citep{takamatsu2012reducing}, co-training
methods~\citep{blum1998combining}, pattern-based
supervision~\citep{gupta2014improved} and feature
annotation~\citep{mann2010generalized,liang2009learning}.
Weak supervision frameworks often train in two stages---first modeling source accuracies to generate weak labels, and then fine-tuning a powerful end model for generalization~\citep{Ratner18, bach2018snorkel,
khetan2017learning, sheng2020gmail, fu2020fast,
zhan2019sequentialws, safranchik2020weakly, boecking2019pairwise}.
Our work removes the second stage from the equation and addresses two common challenges in weak supervision, coarse accuracy modeling and low coverage.

One weak supervision work that does not train in two stages, and models source qualities in a way that can be nonuniform over the points is WeaSuL~\citep{cachay2021endtoend}. However, this capability is present in a different context: end-to-end training of a weak supervision label model with an end model. This prevents the use of the label model directly for prediction, as we seek to do in our work. It requires much heavier computational budget, for example, when training a deep model, which is not needed with our approach. In addition, WeaSuL relies on the use of an encoder for source qualities, rendering a theoretical analysis intractable. By contrast, our approach offers clean and easy-to-interpret theoretical guarantees.

Another recent work that utilizes pre-trained embeddings in weak supervision is from~\citep{lang2022training}. Their work uses the smoothness of the embedding space to remove low-quality weakly-labeled data points, while we use it to improve coverage and fine-grained estimation to weakly label additional points more accurately.

\textbf{Transfer learning} uses large datasets to learn useful feature
representations that can be fine-tuned for downstream tasks~\citep{kolesnikov2020big,
devlin2018bert}.
Transfer learning techniques for text applications typically pre-train on large
corpora of unlabeled data~\citep{devlin2018bert,
brown2020language, radford2019language}, while common
applications of transfer learning to computer
vision pre-train on both large supervised datasets such as
ImageNet~\citep{ILSVRC15} and large
unsupervised or weakly-supervised datasets~\citep{he2019momentum, chen2020simple, radford2021learning}.
Pre-trained embeddings have also been used as data point descriptors for 
kNN search algorithms to improve model
performance, interpretability, and robustness~\citep{papernot2018deep, khandelwal2019generalization}.
We view our work as complementary to these approaches, presenting another
mechanism for using pre-trained networks.

\textbf{Foundation models} offer a new interface for the transfer learning setting: when it is impossible to fine-tune the original models~\citep{Bommasani21FM}.
In this setting, the foundation models can still be used either by direct prompting~\citep{lester2021power, brown2020language}, or by using embeddings~\citep{neelakantan2022text}. 
\cite{wang-etal-2021-want-reduce} prompts FMs to produce pseudolabels, providing a complementary way to use FMs in weak supervision.
In contrast, our work focuses on using FM embeddings.
Since we can only access the final embeddings of some foundation models, we focus on adapters~\citep{houlsby2019parameter} over the final layer in this work---which are equivalent to linear probes~\citep{alain2016understanding}.

\textbf{Semi-supervised and few-shot learning} approaches aim to learn good
models for downstream tasks given a few labeled examples.
Semi-supervised approaches like label propagation~\cite{iscen2019label} start from a few labeled examples and iteratively fine-tune representations on progressively larger datasets, while few-shot learning
approaches such as meta-learning and metric learning aim to build networks that
can be directly trained with a few labels~\cite{snell2017prototypical}.
Our work is inspired by these approaches for expanding signal from a subset of
the data to the entire dataset using FM representations, but we do not assume that our labeling sources are perfect, and we do not tune the representation.

\section{Glossary} \label{sec:supp_glossary}

The glossary is given in Table~\ref{table:glossary} below.

\begin{table*}[bp!]
\centering
\small
\begin{tabular}{l l}
\toprule
Symbol & Used for \\
\midrule
$x$ & Input data point $x \in \X$. \\
$y$ & True task label $y \in \Y = \{-1, +1\}$. \\
$\bm{\lf}$ & Weak sources $\bm{\lf} = \{\lf_1, \dots, \lf_m\}$, where each $\lf_j: \X \rightarrow \Y \cup \{0\}$ is a probabilistic labeling function \\
 & that votes on each $x$. \\
$m$ & Number of weak sources. \\
$f$ & A fixed mapping from input space $\X$ to embedding space $\Z$ that is made available by the off-the-shelf \\
 & foundation model. \\
$\rho$ & A fixed metric on the embedding space, $\rho: \Z \times \Z \rightarrow \mathbb{R}^+$. \\
$\D$ & A training dataset of $n$ i.i.d. unlabeled points, $\D = \{x_i\}_{i = 1}^n$. \\
$n$ & Number of points in the unlabeled training dataset $\D$. \\
$G$ & The dependency graph $G = (V, E)$ used to model $\Pr(y, \bm{\lf} | x)$, where $V = y \cup \bm{\lf}$ and $E$ contains \\
 & edges between $y$ and $\bm{\lf}$. \\
$\Theta(x)$ & The set of canonical parameters $\Theta(x) = \{\theta_y(x), \theta_i(x), \theta_{i,0}(x) \; \forall i \in [m] \}$ corresponding to \\
 & class balance, source accuracy, and the abstain rate used to parametrize $\Pr(y, \bm{\lf} | x)$ in~\eqref{eq:pgm}. \\
$Z$ & Partition function used for normalizing the distribution of $\Pr(y, \bm{\lf} | x)$. \\
$a_i(x)$ & Accuracy parameter of $\lf_i$ on point $x$, $a_i(x) = \E{}{\lf_i y | \lf_i \neq 0, x}$. \\
$\C$ & Partition of the embedding space $\Z$ into nonoverlapping subsets, $\C = \{C_1, \dots, C_s\}$. \\
$s$ & Size of the partition $\C$. \\
$n'$ & The number of points from $\D$ in each subset $C_j$, $n' = \frac{n}{s}$. \\
$C(x)$ & The subset that $x$ belongs to, i.e. $C(x) = C_j$ if $f(x) \in C_j$. \\
$a_i(C(x))$ & Local accuracy parameter of $\lf_i$ on subset $C(x)$, $a_i(C(x)) = \E{}{\lf_i y | \lf_i \neq 0, C(x)}$. \\
$\hat{a}_i(C(x))$ & Our local accuracy estimate of $a_i(C(x))$ using the triplet method in Algorithm~\ref{alg:triplet}. \\
$\bm{\lfbar}$ & Set of extended weak sources, where each $\lfbar_i$ is extended from $\lf_i$ using threshold radius $r_i$ in~\eqref{eq:extended}.\\
$r_i$ & Threshold radius for $\lf_i$, which determines how much beyond the support of $\lf_i$ to extend votes to. \\
$L(\bm{\lf})$ & Generalization error (cross-entropy loss) of the label model, defined as $L(\bm{\lf}) = \E{\D, x, y, \bm{\lf}}{- \log \hat{\Pr}(y | \bm{\lf}, x)}$. \\
$K_y, K_{\lf}, K_{\lf, 0}$ & Constants in Definition~\ref{assumption:lipschitzness} corresponding to label, source, and coverage Lipschitzness, respectively. \\
$\alpha$ & The maximum average inverse source coverage over the subsets, $\alpha = \max_i \E{x}{\frac{1}{p_{ij}} \; \big| \; p_{ij} \neq 0}$, \\
& where $p_{ij} = \Pr(\lf_i \neq 0 | f(x) \in C_j)$ is the coverage of $\lf_i$ on $C_j$. \\
$a_{\max}$ & The maximum source accuracy over the subsets, $a_{\max} = \max_{i, j} a_i(C_j)$. \\
$b_{\min}$ & The minimum rate of agreement between sources over the subsets, \\
 & $b_{\min} = \min_{i, j, k} \{\E{}{\lf_i \lf_j | \lf_i \wedge \lf_k \neq 0, C_j}, \Ehat{\lf_i \lf_k | \lf_i \wedge \lf_k \neq 0, C_j} \}$. \\
$d_{C_j}$ & The diameter of $C_j$, $d_{C_j} = \max_{f(x), f(x') \in C_j} \rho(f(x), f(x'))$. \\
$d_{\C}$ & The average subset diameter $d_{\C} = \E{x}{d_{C(x)}}$. \\
$H(y | \bm{\lf}, x)$ & Conditional entropy of $y$ given $\bm{\lf}, x$. \\
$a_i$ & The average accuracy of $\lf_i$, $a_i = \E{}{\lf_i y | \lf_i \neq 0}$. \\
$\bar{a}_i(r_i)$ & The average accuracy of $\lfbar_i$ on the extended region, $\bar{a}_i(r_i) = \E{}{\lfbar_i y | \lfbar_i \neq 0, \lf_i = 0}$. \\
$\p_{\lf_i}$ & The distribution of $(x, y)$ over the support of $\lf_i$, $\p_{\lf_i} = \Pr( \cdot | \lf_i \neq 0)$.\\
$M$ & An increasing function $M: \mathbb{R}^+ \rightarrow [0, 1]$ used to describe probabilistic Lipschitzness. \\
$\beta_i$ & $\lf_i$'s accuracy over an area close to where $\lf_i$ is extended and $y$ changes value, \\
 & $\beta_i = \E{}{\lf_i y | \lf_i \neq 0, \exists(x', y') : \lf_i(x') = 0, \rho(f(x), f(x')) \le r_i, y' = y}$. \\
$p_i$ & The proportion of the region where $\lfbar_i$ is extended, $p_i = \Pr(\lf_i \neq 0, \lf_i = 0)$. \\
$p(\lf_{-i})$ & The label model's true probability of outputting the correct label in the extension region when \\
 &  only using $\lf_{-i} = \bm{\lf} \backslash \lf_i$, $p(\lf_{-i}) = \E{y', \lf_{-i}, \lfbar_i \neq 0, \lf_i = 0}{\Pr(y = y' | \lf_{-i}, x)}$. \\
\toprule
\end{tabular}
\caption{
	Glossary of variables and symbols used in this paper.
}
\label{table:glossary}
\end{table*}

\newpage

\section{Additional Algorithmic Details}\label{sec:supp_alg}

We describe some properties of the graphical model that justify our algorithm (Section~\ref{sec:supp_pgm}). Then, we formalize the triplet method algorithm for estimating local accuracy parameters, $\hat{a}_i(C_j)$ (Section~\ref{sec:supp_triplet}).

\subsection{Properties of the Graphical Model} \label{sec:supp_pgm}

\begin{lemma}
For $x, y, \bm{\lf}$ satisfying~\eqref{eq:pgm}, it holds for any $\lf_i$ that
\begin{align*}
\Pr(y, \lf_i = 0 | x) = \Pr(y | x) \Pr(\lf_i = 0 | x).
\end{align*}

That is, $y \independent \ind{\lf_i = 0} | x$.
\label{lemma:abstain}
\end{lemma}

\begin{proof}
Denote $\lf_{-i} = \bm{\lf} \backslash \lf_i$, and equivalently let $\theta_{-i}$ and $\theta_{-i, 0}$ denote vectors of canonical parameters corresponding to $\lf_{-i}$ in~\eqref{eq:pgm}. We show independence by proving that $\Pr(y = 1, \lf_i = 0 | x) = \Pr(y = 1 | x) \Pr(\lf_i = 0 | x)$:
\begin{align}
\Pr(y = 1, \lf_i = 0 | x) &= \frac{1}{Z} \sum_{\lf_{-i}} \exp \big(\theta_y(x) + \theta_{i, 0}(x) + \theta_{-i}(x) \lf_{-i} + \theta_{-i, 0}(x) \ind{\lf_{-i} = 0} \big) \label{eq:y1lfi0} \\
&= \frac{1}{Z}\exp(\theta_y(x) + \theta_{i, 0}(x))  \sum_{\lf_{-i}} \exp \big(\theta_{-i}(x) \lf_{-i} + \theta_{-i, 0}(x) \ind{\lf_{-i} = 0} \big). \nonumber 
\end{align}

$\Pr(y = 1 | x)$ can be written as $\Pr(y = 1, \lf_i = 1 | x) + \Pr(y = 1, \lf_i = -1 | x) + \Pr(y = 1, \lf_i = 0 | x)$:
\begin{align}
&\Pr(y = 1 | x) = \frac{1}{Z} \sum_{\lf_{-i}} \Big(\exp(\theta_y(x) + \theta_i(x) + \theta_{-i}(x) \lf_{-i} + \theta_{-i, 0}(x) \ind{\lf_{-i} = 0}) \label{eq:y1} \\
&\qquad \qquad + \exp(\theta_y(x) - \theta_i(x) + \theta_{-i}(x) \lf_{-i} + \theta_{-i, 0}(x) \ind{\lf_{-i} = 0}) \nonumber \\
&\qquad \qquad + \exp(\theta_y(x) + \theta_{i, 0}(x) + \theta_{-i}(x) \lf_{-i} + \theta_{-i, 0}(x) \ind{\lf_{-i} = 0}) \Big) \nonumber \\
&= \frac{1}{Z} \sum_{\lf_{-i}} \exp(\theta_{-i}(x) \lf_{-i} + \theta_{-i, 0}(x) \ind{\lf_{-i} = 0}) \exp(\theta_y(x)) \big( \exp(\theta_i(x)) + \exp(- \theta_i(x)) + \exp(\theta_{i, 0}(x))\big). \nonumber
\end{align}

$\Pr(\lf_i = 0 |x)$ can be written as $\Pr(y = 1, \lf_i = 0 | x) + \Pr(y = -1, \lf_i = 0 | x)$:
\begin{align}
\Pr(\lf_i = 0 | x) &= \frac{1}{Z} \sum_{\lf_{-i}} \Big(\exp(\theta_y(x) + \theta_{i, 0}(x) + \theta_{-i}(x) \lf_{-i} + \theta_{-i, 0}(x) \ind{\lf_{-i} = 0}) \label{eq:lfi0}\\
&+ \exp(-\theta_y(x) + \theta_{i, 0}(x) - \theta_{-i}(x) \lf_{-i} + \theta_{-i, 0}(x) \ind{\lf_{-i} = 0}) \Big). \nonumber 
\end{align}

Setting~\eqref{eq:y1lfi0} equal to~\eqref{eq:y1} times~\eqref{eq:lfi0}, the term $\sum_{\lf_{-i}} \exp(\theta_{-i}(x) \lf_{-i}  + \theta_{-i, 0}(x) \ind{\lf_{-i} = 0})$ in the former two equations cancels out. We thus aim to prove the following equality:
\begin{align}
Z \exp(\theta_y(x) &+ \theta_{i, 0}(x)) = \exp(\theta_y(x)) \Big(\exp(\theta_i(x)) + \exp(-\theta_i(x)) + \exp(\theta_{i, 0}(x)) \Big) \times \label{eq:abstain_indep}\\
&\sum_{\lf_{-i}} \exp(\theta_{-i, 0} \ind{\lf_{-i} = 0}) \exp(\theta_{i, 0}(x)) \Big(\exp(\theta_y(x) + \theta_{-i}(x) \lf_{-i}(x)) + \exp(- \theta_y(x) - \theta_{-i}(x) \lf_{-i}(x)) \Big). \nonumber 
\end{align}

Canceling out $\exp(\theta_y(x) + \theta_{i, 0}(x))$, \eqref{eq:abstain_indep} is equal to
\begin{align*}
Z &=  \Big(\exp(\theta_i(x)) + \exp(-\theta_i(x)) + \exp(\theta_{i, 0}(x)) \Big) \times \\
& \sum_{\lf_{-i}} \exp(\theta_{-i, 0} \ind{\lf_{-i} = 0}) \Big(\exp(\theta_y(x) + \theta_{-i}(x) \lf_{-i}(x)) + \exp(- \theta_y(x) - \theta_{-i}(x) \lf_{-i}(x)) \Big),
\end{align*}

which is true since the RHS iterates over all values of $\lf_{-i}, y,$ and $\lf_i$. We have shown that $\Pr(y = 1, \lf_i = 0 | x) = \Pr(y = 1 | x) \Pr(\lf_i = 0 | x)$ and thus that $y \independent \ind{\lf_i = 0} | x$ for any $\lf_i$.

Due to this independence property, we note that
\begin{align*}
\Pr(\lf_i = 0, \lf_{-i} | x) &= \Pr(\lf_i = 0, \lf_{-i} | y = 1, x)\Pr(y = 1 | x) + \Pr(\lf_i = 0, \lf_{-i} | y = -1, x) \Pr(y = -1 | x) \\
&= \Pr(\lf_i = 0 | x) \Big(\Pr(\lf_{-i} | y = 1, x) \Pr(y = 1 | x) + \Pr(\lf_{-i} | y = -1, x) \Pr(y = -1 | x) \Big) \\
&= \Pr(\lf_i = 0 | x) \Pr(\lf_{-i} | x),
\end{align*} 

and hence
\begin{align*}
\Pr(y | \lf_i = 0, \lf_{-i}, x) &= \frac{\Pr(\lf_i = 0 | y, x) \Pr(\lf_{-i} | y, x) \Pr(y | x)}{\Pr(\lf_i = 0, \lf_{-i} | x)} = \frac{\Pr(\lf_i = 0 | x) \Pr(\lf_{-i} | y, x) \Pr(y | x)}{\Pr(\lf_i = 0 | x) \Pr(\lf_{-i} | x)} \\
&= \Pr(y | \lf_{-i}, x).
\end{align*}

\end{proof}

\begin{lemma}
For any $i \neq j$, if $x, \bm{\lf}, y$ follows~\eqref{eq:pgm}, then $\lf_i y \independent \lf_j y | \lf_i \wedge \lf_j \neq 0, x$.
\label{lemma:triplet_independence}
\end{lemma}

\begin{proof}
Conditioning on the event that $\bm{\lf}\neq 0$, we have that
\begin{align}
\Pr(y, \bm{\lf} | \bm{\lf}\neq 0, x) = \frac{1}{Z_0} \exp\Big(\theta_y(x) y + \sum_{i = 1}^m \theta_i(x) \lf_i y \Big) \label{eq:pgm_easy}
\end{align}

for some partition function $Z_0$ different from $Z$ in~\eqref{eq:pgm}. This graphical model now follows the structure of the graphical model in~\cite{fu2020fast} (see their Equation 3). We can thus apply Proposition 1 of their work to get that $\lf_i y \independent \lf_j y | \bm{\lf} \neq 0, x$. From Lemma~\ref{lemma:abstain} and conditional independence of sources, this independence property is equivalent to $\lf_i y \independent \lf_j y | \lf_i \wedge \lf_j \neq 0, x$, as desired.
\end{proof}

\begin{lemma}
If $x, \bm{\lf}, y$ follows~\eqref{eq:pgm}, then for any $\lf_i$,
\begin{align}
\Pr(\lf_i = 1 | y = 1, \lf_i \neq 0, x) &= \Pr(\lf_i = -1 | y = -1, \lf_i \neq 0, x) = \Pr(\lf_i y = 1 | \lf_i \neq 0, x) \\
\Pr(\lf_i = -1 | y = 1, \lf_i \neq 0, x) &= \Pr(\lf_i = 1 | y = -1, \lf_i \neq 0, x) = \Pr(\lf_i y = -1 | \lf_i \neq 0, x). 
\end{align}

Therefore,
\begin{align*}
\Pr(\lf_i | y, \lf_i \neq 0, x) = \frac{1 + \sgn(\lf_i y) a_i(x)}{2}.
\end{align*}
\label{lemma:symmetry}
\end{lemma}

\begin{proof}
Conditioning on the event that $\bm{\lf} \neq 0$, the graphical model is of the form in~\eqref{eq:pgm_easy} above. This graphical model also follows the structure of that in~\cite{chen2021comparing}, and therefore we obtain our desired properties by Lemma 2 of their work.
\end{proof}

\subsection{Local Accuracy Parameter Estimation Algorithm}\label{sec:supp_triplet}

We formalize the triplet method used to recover latent source parameters $\Pr(\lf_i | y, x)$.
First, when we want to evaluate $\Pr(\lf_i = 1 | y, x)$ or $\Pr(\lf_i = -1 | y, x)$, this probability can be written as $\Pr(\lf_i | y, x, \lf_i \neq 0) \Pr(\lf_i \neq 0 | y, x) = \Pr(\lf_i y | x, \lf_i \neq 0) \Pr(\lf_i \neq 0 | x)$ by Lemmas~\ref{lemma:abstain} and~\ref{lemma:symmetry}. We have that $\E{}{\lf_i y | x, \lf_i \neq 0} = \Pr(\lf_i y = 1 | x, \lf_i \neq 0) - \Pr(\lf_i y = -1 | x, \lf_i \neq 0) = 2 \Pr(\lf_i y = 1 | x, \lf_i \neq 0) - 1$, so $\Pr(\lf_i | y, x) =\frac{1 + \sgn(\lf_i y)a_i(x)}{2} \cdot \Pr(\lf_i \neq 0 | x)$ when $\lf_i \in \{-1, 1\}$. When $\lf_i$ is $0$, the probability we want to estimate is $\Pr(\lf_i = 0 | y, x) = \Pr(\lf_i = 0 | x)$ by Lemma~\ref{lemma:abstain}.

We now explain how our algorithm estimates $a_i(x)$. From Lemma~\ref{lemma:triplet_independence}, we have that $\lf_i y \independent \lf_j y | \lf_i \wedge \lf_j \neq 0, x$ for any $i, j$. Then, given any set of $\lf_i, \lf_j, \lf_k$, we have the set of equations
\begin{align*}
a_i(x) a_j(x) &= \E{}{\lf_i \lf_j | \lf_i \wedge \lf_j \neq 0, x} \\
a_i(x) a_k(x) &= \E{}{\lf_i \lf_k | \lf_i \wedge \lf_k \neq 0, x} \\
a_j(x) a_k(x) &= \E{}{\lf_j \lf_k | \lf_j \wedge \lf_k \neq 0, x}.
\end{align*} 

Solving, we get that 
\begin{align*}
|a_i(x)| = \sqrt{\bigg| \frac{\E{}{\lf_i \lf_j | \lf_i \wedge \lf_j \neq 0, x} \E{}{\lf_i \lf_k | \lf_i \wedge \lf_k \neq 0, x}}{\E{}{\lf_j \lf_k | \lf_j \wedge \lf_k \neq 0, x}}\bigg|}.
\end{align*}

This property allows us to recover $a_i(x)$ up to a sign. As discussed in Section~\ref{sec:method}, we use $C(x)$ to estimate the accuracy parameter over a region of the embedding space, such that in fact we are estimating $a_i(C(x)) = \E{}{\lf_i y | \lf_i \neq 0, C(x)}$ (since Lemma~\ref{lemma:triplet_independence} holds on any $x$, it holds conditioned over $C(x)$ too). We resolve the sign of the accuracy parameter by assuming that $a_i(C(x))> 0$, meaning that the accuracy of a source over a subset is better than random. Finally, rather than estimating $a_i(C(x))$ using just one pair of $\lf_j$ and $\lf_k$, we compute the average $a_i(C(x))$ over all other pairs ($\lf_j, \lf_k \in \bm{\lf}\backslash \lf_i$) to make the estimate less noisy. Our approach for computing $\hat{a}_i(C_j)$ for any $\lfbar_i$ and $C_j$ (note that $\lfbar_i$ and $\lf_i$ are interchangeable in the above given that $(x, y, \bm{\lf})$ and $(x, y, \bm{\lfbar})$ both satisfy~\eqref{eq:pgm}) is described in Algorithm~\ref{alg:triplet}.

\begin{algorithm}[t]
	\caption{Local Accuracy Estimation (Triplet Method)}
	\begin{algorithmic}
		\STATE \textbf{Input:}
		Dataset $\D$, weak sources $\bm{\lfbar}$, partition $C_j$.
		\STATE \textbf{Returns:} Estimate of local accuracy $\hat{a}_i(C_j)$.
		\FOR{$k, l \in [m] \backslash i$}
			\STATE Estimate $\Ehat{\lfbar_i \lfbar_k |\lfbar_i \wedge \lfbar_k \neq 0, C_j}$ over the set of points $\{x \in \D: \lfbar_i(x), \lfbar_k(x) \neq 0, f(x) \in C_j\}$, and similarly estimate $\Ehat{\lfbar_i \lfbar_l | \lfbar_i \wedge \lfbar_l \neq 0, C_j}$ and $\Ehat{\lfbar_k \lfbar_l | \lfbar_k \wedge \lfbar_l \neq 0, C_j}$.
			\STATE Compute $\hat{a}_i^{k, l}(C_j) = \sqrt{\bigg|\frac{\Ehat{\lfbar_i \lfbar_k | \lfbar_i \wedge \lfbar_k \neq 0, C_j} \Ehat{\lfbar_i \lfbar_l | \lfbar_i \wedge \lfbar_l \neq 0, C_j}}{\Ehat{\lfbar_k \lfbar_l | \lfbar_k \wedge \lfbar_l \neq 0, C_j}} \bigg|}$.  
		\ENDFOR
		\RETURN $\hat{a}_i(C_j)$ as the average over all $\hat{a}_i^{k, l}(C_j)$.
	\end{algorithmic}
	\label{alg:triplet}
\end{algorithm}

\section{Proofs} \label{sec:supp_proofs}

We present the proofs for our results in Section~\ref{sec:theory}.

\subsection{Proofs for Section~\ref{subsec:gen_err}}

The proof of Theorem~\ref{thm:gen_err} involves decomposing the generalization error into the irreducible error, bias from using $C(x)$, and variance (sampling error).

\generr*

\begin{proof}

We can write the generalization error as
\begin{align*}
L(\bm{\lf}) = \E{\D, x, y, \bm{\lf}}{-\log \hat{\Pr}(y | \bm{\lf}, x)} =  \E{}{-\log \frac{\hat{\Pr}(y | \bm{\lf}, x)}{\Pr(y | \bm{\lf}, x)}} - \E{x, y, \bm{\lf}}{\log \Pr(y | \bm{\lf}, x)}.
\end{align*}

$-\E{x, y, \bm{\lf}}{\Pr(y | \bm{\lf}, x)}$ is equal to the conditional entropy of $y$ given $\bm{\lf}, x$, expressed as $H(y | \bm{\lf}, x)$observing. This describes the entropy of $y$ after observing the weak labels and input and thus depends on how much signal we are getting from the labelers. Next, we decompose the expected log ratio using our construction of $\hat{\Pr}(\lf_i | y, C(x))$ to get
\begin{align}
L(\bm{\lf}) &=\E{}{-\log \bigg(\frac{\prod_{i = 1}^m \hat{\Pr}(\lf_i | y, C(x)) \Pr(y | C(x))}{\hat{\Pr}(\bm{\lf} | C(x))} \cdot \frac{\Pr(\bm{\lf} | x)}{\prod_{i = 1}^m \Pr(\lf_i | y, x) \Pr(y | x)} \bigg)} + H(y | \bm{\lf}, x) \nonumber \\
&= -\E{}{\sum_{i = 1}^m \log \frac{\hat{\Pr}(\lf_i | y, C(x))}{\Pr(\lf_i | y, x)}} - \E{}{\log \frac{\Pr(\bm{\lf} | x)}{\hat{\Pr}(\bm{\lf} | C(x))}} - \E{}{\log \frac{\Pr(y | C(x))}{\Pr(y | x)}} + H(y | \bm{\lf}, x) \nonumber \\
&= -\E{}{\sum_{i = 1}^m \log \frac{\hat{\Pr}(\lf_i | y, C(x))}{\Pr(\lf_i | y, x)}} - \E{x}{D_{\text{KL}}(\Pr(\bm{\lf} | x) || \hat{\Pr}(\bm{\lf} | C(x)))} - \E{}{\log \frac{\Pr(y | C(x))}{\Pr(y | x)}} + H(y | \bm{\lf}, x) \nonumber \\
&\le \sum_{i = 1}^m \E{x, y, \lf_i}{\log \frac{\Pr(\lf_i | y, x) }{\hat{\Pr}(\lf_i | y, C(x)) }} - \E{x, y}{\log \frac{\Pr(y | C(x))}{\Pr(y | x)}} + H(y | \bm{\lf}, x), \label{eq:gen_bound}
\end{align}

where we have used Lemma~\ref{lemma:abstain} and the fact that the Kullback-Leibler divergence is always nonnegative in the last line. For notation, let $\text{KL}_{C(x)}(y) = \E{x}{D_{\text{KL}}(\Pr(y | x) || \Pr(y | C(x)))} = \E{x, y}{\log \frac{\Pr(y | x)}{\Pr(y | C(x))}}$, be the KL-divergence between distributions conditioned on $C(x)$ versus $x$, which describes the bias we incur from using a partition. Then, $L(\bm{\lf}) \le \sum_{i = 1}^m \E{x, y, \lf_i}{\log \frac{\Pr(\lf_i | y , x)}{\hat{\Pr}(\lf_i | y, C(x))}} + \text{KL}_{C(x)}(y) + H(y | \bm{\lf}, x)$.

 We now simplify the expression $\E{}{\log \frac{\Pr(\lf_i | y, x) }{\hat{\Pr}(\lf_i | y, C(x)) }}$ based on if $\lf_i = 0$ or $\lf_i \in \{-1, 1\}$:
\begin{align}
\mathbb{E}\bigg[&\log \frac{\Pr(\lf_i | y, x) }{\hat{\Pr}(\lf_i | y, C(x)) }\bigg] = \E{x}{\Pr(\lf_i = 0 | x) \log \frac{\Pr(\lf_i = 0 | x)}{\hat{\Pr}(\lf_i = 0 | C(x))}} + \E{x, y, \lf_i \neq 0}{\Pr(\lf_i \neq 0 | x) \log \frac{\Pr(\lf_i | y, x)}{\hat{\Pr}(\lf_i | y, C(x))}} \nonumber \\
&= \E{x}{\Pr(\lf_i = 0 | x) \log \frac{\Pr(\lf_i = 0 | x)}{\hat{\Pr}(\lf_i = 0 | C(x))} + \Pr(\lf_i \neq 0 | x) \log \frac{\Pr(\lf_i \neq 0| x)}{\hat{\Pr}(\lf_i \neq 0 | C(x))}} \nonumber \\
&+ \E{x, y, \lf_i \neq 0}{\Pr(\lf_i \neq 0 | x) \log \frac{\Pr(\lf_i | y, x, \lf_i \neq 0)}{\hat{\Pr}(\lf_i | y, C(x), \lf_i \neq 0)}} \nonumber \\
&= \E{x}{D_{\text{KL}}( \Pr(z_i | x) || \hat{\Pr}(z_i | C(x)))} + \E{x, y, \lf_i \neq 0}{\Pr(\lf_i \neq 0 | x) \log \frac{\Pr(\lf_i | y, x, \lf_i \neq 0)}{\hat{\Pr}(\lf_i | y, C(x), \lf_i \neq 0)}},
\label{eq:lf_i_gen_err}
\end{align}

where $z_i = \ind{\lf_i = 0}$ is an indicator variable pertaining to coverage. The first KL divergence pertains to estimating the coverage of $\lf_i$, while the second pertains to estimating the accuracy parameter of $\lf_i$. 
The first term in~\eqref{eq:lf_i_gen_err} can be written as
\begin{align*}
&\mathbb{E}_x [D_{\text{KL}}( \Pr(z_i | x) || \hat{\Pr}(z_i | C(x)))] = \text{KL}_{C(x)}(z_i) \\
& + \E{x}{\Pr(\lf_i = 0 | x) \log \frac{\Pr(\lf_i = 0 | C(x))}{\hat{\Pr}(\lf_i = 0 | C(x))} + \Pr(\lf_i \neq 0 | x) \log \frac{\Pr(\lf_i \neq 0| C(x))}{\hat{\Pr}(\lf_i \neq 0 | C(x))}} \\
&= \text{KL}_{C(x)}(z_i) + \E{C(x)}{\Pr(\lf_i = 0 | C(x)) \log \frac{\Pr(\lf_i = 0 | C(x))}{\hat{\Pr}(\lf_i = 0 | C(x))} + \Pr(\lf_i \neq 0 | C(x)) \log \frac{\Pr(\lf_i \neq 0 | C(x))}{\hat{\Pr}(\lf_i \neq 0 | C(x))}} \\
&= \text{KL}_{C(x)}(z_i) + \E{C(x), z_i}{\log \frac{\Pr(z_i | C(x))}{\hat{\Pr}(z_i | C(x))}},
\end{align*}

The second term in~\eqref{eq:lf_i_gen_err} can be written as
\begin{align*}
\E{}{\Pr(\lf_i \neq 0 | x) \log \frac{\Pr(\lf_i | y, x, \lf_i \neq 0)}{\hat{\Pr}(\lf_i | y, C(x), \lf_i \neq 0)}}  &\le \E{x, y, \lf_i \neq 0}{\log \bigg(\frac{\Pr(\lf_i | y, x, \lf_i \neq 0)}{\Pr(\lf_i | y, C(x), \lf_i \neq 0)} \cdot  \frac{\Pr(\lf_i | y, C(x), \lf_i \neq 0)}{\hat{\Pr}(\lf_i | y, C(x), \lf_i \neq 0)}\bigg)} \\
&= \text{KL}_{C(x)}(\lf_i | y, \lf_i \neq 0) + \E{x, y, \lf_i \neq 0}{\log  \frac{\Pr(\lf_i | y, C(x), \lf_i \neq 0)}{\hat{\Pr}(\lf_i | y, C(x), \lf_i \neq 0)}}.
\end{align*}

Putting everything together in~\eqref{eq:gen_bound}, the generalization error is at most
\begin{align*}
L(\bm{\lf}) &\le \sum_{i = 1}^m  \bigg(\E{C(x), z_i}{\log \frac{\Pr(z_i | C(x))}{\hat{\Pr}(z_i | C(x))}} + \E{x, y, \lf_i \neq 0}{\log  \frac{\Pr(\lf_i | y, C(x), \lf_i \neq 0)}{\hat{\Pr}(\lf_i | y, C(x), \lf_i \neq 0)}} \\
&+ \text{KL}_{C(x)}(z_i) + \text{KL}_{C(x)}(\lf_i | y, \lf_i \neq 0)\bigg) +  \text{KL}_{C(x)}(y) + H(y | \bm{\lf}, x).
\end{align*}

We can interpret the generalization error as consisting of bias, variance (and irreducible error) coming from 1) estimating the coverage of a weak source over a part, and then, conditioned on the support of a source, 2) estimating the accuracy of the source over a part. The bias is from using $C(x)$ instead of $x$, and the variance is from estimating over the dataset over these two steps.

Using Lemmas~\ref{lemma:coverage_estimation},~\ref{lemma:accuracy_estimation}, and~\ref{lemma:kl_biases} we get our desired bound. 
\end{proof}

\begin{lemma}
The sampling error term coming from estimating $\lf_i$'s coverage, $\E{C(x), z_i}{\log \frac{\Pr(z_i | C(x))}{\hat{\Pr}(z_i | C(x))}}$, where $z_i = \ind{\lf_i = 0}$, is equal to
\begin{align*}
\E{C(x), z_i}{\log \frac{\Pr(z_i | C(x))}{\hat{\Pr}(z_i | C(x))}} = \frac{s}{n} + o(1/n).
\end{align*}
\label{lemma:coverage_estimation}
\end{lemma}

\begin{proof}
We can write this expectation across each $C_j$. Denote $p_{ij} = \Pr(\lf_i \neq 0 | C_j)$ as $\lf_i$'s coverage on $C_j$, and equivalently $\phat_{ij}$ as its estimate over $\D$. Then,
\begin{align}
\E{C(x), z_i}{\log \frac{\Pr(z_i | C(x))}{\hat{\Pr}(z_i | C(x))}} = \sum_{j = 1}^s \Pr(f(x) \in C_j) \E{\D}{p_{ij} \log \frac{p_{ij}}{\phat_{ij}} + (1 - p_{ij})\log \frac{1 - p_{ij}}{1 - \phat_{ij}}}. \label{eq:coverage_parts}
\end{align}

Performing a Taylor approximation of $g(x) = \log \frac{c}{x}$ at $x=c$ gives us $\log \frac{c}{x} \approx \log 1 + -\frac{1}{c}(x - c) + \frac{1}{2c^2}(x - c)^2$. Setting $x = \phat_{ij}, 1 - \phat_{ij}$ and $c = p_{ij}, 1 - p_{ij}$ respectively in~\eqref{eq:coverage_parts} and using the fact that $\phat_{ij}$ is an unbiased estimate of $p_{ij}$, this expression becomes
\begin{align*}
\E{C(x), z_i}{\log \frac{\Pr(z_i | C(x))}{\hat{\Pr}(z_i | C(x))}} &= \sum_{j = 1}^s \Pr(f(x) \in C_j) \frac{1}{p_{ij}(1 - p_{ij})} \E{}{(p_{ij} - \phat_{ij})^2} + o(1/n) \\
&= \sum_{j = 1}^s \Pr(f(x) \in C_j) \frac{1}{p_{ij}(1 - p_{ij})} \Var{}{\phat_{ij}} + o(1/n),
\end{align*}

where we use the fact that the Taylor remainder scales in $\E{}{(\phat_{i, j} - p_{i, j})^3 \big| C_j} \sim \mathcal{O}(1/n^2)$. We can simplify the variance
$\Var{}{\phat_{i, j}} = \Var{}{\frac{1}{n'} \sum_{x: f(x) \in C_j} \ind{\lf_i(x) \neq 0}} = \frac{1}{(n')^2} \sum_{x: f(x) \in C_j}  \Var{}{\ind{\lf_i(x) \neq 0}} = \frac{p_{i, j}(1 - p_{i, j})}{n'}$. Putting this all together, we have
\begin{align*}
\E{C(x), z_i}{\log \frac{\Pr(z_i | C(x))}{\hat{\Pr}(z_i | C(x))}} = \sum_{j = 1}^s \Pr(f(x) \in C_j) \frac{1}{n'} + o(1/n) = \frac{s}{n} + o(1/n).
\end{align*}

\end{proof}

\begin{lemma}
Define $p_{ij} = \Pr(\lf_i \neq 0 | C_j)$ as the coverage of the $\lf_i$ on $C_j$. The sampling error term coming from estimating source accuracy of $\lf_i$, $\Pr(\lf_i| y, C(x), \lf_i \neq 0)$, is at most
\begin{align*}
\E{x, y, \lf_i \neq 0}{\log  \frac{\Pr(\lf_i | y, C(x), \lf_i \neq 0)}{\hat{\Pr}(\lf_i | y, C(x), \lf_i \neq 0)}} \le \E{C_j}{\frac{1}{p_{ij}} \; \Big| \; p_{ij} \neq 0} \cdot \frac{3 s }{8 n} \cdot \frac{1 - \bmin^2}{\bmin^2 (1 - a_{\max}^2)} \bigg(\frac{1}{\bmin^4} + \frac{2}{\bmin^2} \bigg) + o(1/n).
\end{align*}
\label{lemma:accuracy_estimation}
\end{lemma}

\begin{proof}
Define $\mathcal{C}_i \subseteq \mathcal{C}$ to be the subsets where $\lf_i$ has non-zero coverage, $\{C \in \mathcal{C}: \exists x: f(x) \in C, \lf_i(x) \neq 0 \}$. When there are subsets with no $\lf_i$ coverage, we do not estimate the accuracy and can discard them from this bound. We can thus write the above expectation as $\E{}{\log  \frac{1 + \sgn(\lf_i y) \cdot a_i(C(x))}{1 + \sgn(\lf_i y) \cdot \hat{a}_i(C(x))}} = \sum_{C_j \in \mathcal{C}_i} \Pr(f(x) \in C_j) \E{}{\log \frac{1 + \sgn(\lf_i y) \cdot a_i(C_j)}{1 + \sgn(\lf_i y) \cdot \hat{a}_i(C_j)} \Big| C_j}$. We can decompose the expectation as
\begin{align}
 \E{x, y, \lf_i \neq 0}{\log \frac{1 + \sgn(\lf_i y) \cdot a_i(C_j)}{1 + \sgn(\lf_i y) \cdot \hat{a}_i(C_j)} \bigg| C_j} &= \E{}{\log \frac{1 + a_i(C_j)}{1 + \hat{a}_i(C_j)} \bigg| C_j} \Pr(\lf_i y \texttt{=} 1 | C_j, \lf_i \neq 0)\\
 &+\E{}{\log \frac{1 - a_i(C_j)}{1 - \hat{a}_i(C_j)} \bigg| C_j} \Pr(\lf_i y \texttt{=} -1 | C_j, \lf_i \neq 0). \label{eq:param_err_decomposition}
\end{align}

$\Pr(\lf_i y = 1 | C_j, \lf_i \neq 0)$ is equal to $\frac{1 + a_i(C_j)}{2}$. \eqref{eq:param_err_decomposition} becomes
\begin{align}
 \E{}{\log \frac{1 + \sgn(\lf_i y) \cdot a_i(C_j)}{1 + \sgn(\lf_i y) \cdot \hat{a}_i(C_j)} \bigg| C_j} &\texttt{=} \frac{1}{2} \bigg((1 + a_i(C_j)) \E{}{\log \frac{1 + a_i(C_j)}{1 + \hat{a}_i(C_j)} \bigg| C_j} + (1 - a_i(C_j)) \E{}{\log \frac{1 - a_i(C_j)}{1 - \hat{a}_i(C_j)} \bigg| C_j}  \bigg). \label{eq:param_err_decomposition2}
\end{align}

Again, we can perform a Taylor expansion on $g(x) = \log \frac{1 + c}{1 + x}$ at $x = c$ to get that $\log \frac{1 + c}{1 + x} \approx -\frac{1}{1+c}(x - c) + \frac{1}{2(1+c)^2}(x - c)^2$, and therefore $\E{}{\log \frac{1 + a_i(C_j)}{1 + \hat{a}_i(C_j)} \Big| C_j} = \frac{\E{}{a_i(C_j) - \hat{a}_i(C_j)}}{1 + a_i(C_j)} + \frac{\E{}{(\hat{a_i}(C_j) - a_i(C_j))^2}}{2(1 + a_i(C_j))^2} + o(1/n)$, (see Lemma 4 of~\cite{chen2021comparing} for bounding the Taylor remainder). Similarly, we have that $\E{}{\log \frac{1 - a_i(C_j)}{1 - \hat{a}_i(C_j)} \Big| C_j} = \frac{\E{}{\hat{a}_i(C_j) - a_i(C_j)}}{1 - a_i(C_j)} + \frac{\E{}{(\hat{a}_i(C_j) - a_i(C_j))^2}}{2(1 - a_i(C_j))^2} + o(1/n)$. Therefore,~\eqref{eq:param_err_decomposition2} becomes
\begin{align*}
 \E{}{\log \frac{1 + \sgn(\lf_i y) \cdot a_i(C_j)}{1 + \sgn(\lf_i y) \cdot \hat{a}_i(C_j)} \bigg| C_j} &= \frac{1}{2} \bigg(\E{}{a_i(C_j) - \hat{a}_i(C_j)} + \frac{\E{}{(\hat{a}_i(C_j) - a_i(C_j))^2}}{2(1 + a_i(C_j))} \\
 &+ \E{}{\hat{a}_i(C_j) - a_i(C_j)} + \frac{\E{}{(\hat{a}_i(C_j) - a_i(C_j))^2}}{2(1 - a_i(C_j))} \bigg) + o(1/n) \\
 &= \frac{1}{2} \cdot \frac{\E{}{(\hat{a}_i(C_j) - a_i(C_j))^2}}{1 - a_i(C_j)^2} + o(1/n).
\end{align*}

The value of $\E{}{(\hat{a}_i(C_j) - a_i(C_j))^2}$ has been studied in previous works that use the triplet method of~\cite{fu2020fast}. In particular, we use Lemma 6 of~\cite{chen2021comparing} to get that 
\begin{align*}
\E{}{(\hat{a}_i(C_j) - a_i(C_j))^2} \le \frac{3s }{4 p_{i, j} n} \cdot \frac{1 - \bmin^2}{\bmin^2} \bigg(\frac{1}{\bmin^4} + \frac{2}{\bmin^2} \bigg).
\end{align*}

 Therefore, the overall expression can be bounded by
\begin{align*}
\E{}{\log \frac{1 + \sgn(\lf_i y) \cdot a_i(C(x)) }{1 + \sgn(\lf_i y) \cdot \hat{a}_i(C(x)) }} &\le \sum_{C_j \in \mathcal{C}_i} \Pr(f(x) \in C_j) \frac{1}{2(1 - a_{\max}^2)} \cdot \frac{3s}{4 p_{i, j} n} \cdot \frac{1 - \bmin^2}{\bmin^2} \bigg(\frac{1}{\bmin^4} + \frac{2}{\bmin^2} \bigg) + o\Big(\frac{1}{n}\Big)\\
&\le \E{C_j}{\frac{1}{p_{ij}} \; \Big| \; p_{ij} \neq 0} \cdot \frac{3 s }{8 n} \cdot \frac{1 - \bmin^2}{\bmin^2 (1 - a_{\max}^2)} \bigg(\frac{1}{\bmin^4} + \frac{2}{\bmin^2} \bigg) + o\Big(\frac{1}{n} \Big).
\end{align*}

\end{proof}

\begin{lemma}
Denote $z_i = \ind{\lf_i = 0}$ and $\text{KL}_{C(x)}(\cdot) = \E{x}{D_{\text{KL}}( \Pr(\cdot | x), \Pr(\cdot | C(x)))}$. The bias terms from conditioning on $C(x)$ rather than $x$ are at most
\begin{align*}
&\text{KL}_{C(x)}(y) \le 2 K_y d_{\C} \\
&\text{KL}_{C(x)}(\lf_i | y, \lf_i \neq 0) \le 2 m K_{\lf} d_{\C} \\
&\text{KL}_{C(x)}(z_i) \le 2 m K_{\lf, 0} d_{\C}.
\end{align*}
\label{lemma:kl_biases}
\end{lemma}

\begin{proof}
We can write the expected KL-divergence between the distribution of the true label $y$ conditioned on $C(x)$ versus $x$ as 
\begin{align}
\text{KL}_{C(x)}(y) = \E{x}{D_{\text{KL}}(\Pr(y|x) || \Pr(y | C(x)))} = \sum_{j = 1}^s \Pr(f(x) \in C_j) \int \Pr(x | C_j) D_{\text{KL}}( \Pr(y | x, C_j) || \Pr(y | C_j)) dx. \label{eq:y_kl}
\end{align} 

This inner KL-divergence is on two Bernoulli distributions. Define $p_{y, j} = \Pr(y = 1 |C_j)$, and denote $p_{y, x, j} = \Pr(y = 1 | x, C_j)$. Then, $D_{\text{KL}}(\Pr(y | x, C_j) || \Pr(y | C_j)) = p_{y, x, j} \log \frac{p_{y, x, j}}{p_{y, j}} + (1 - p_{y, x, j}) \log \frac{1 - p_{y, x, j}}{1 - p_{y, j}}$. 

Next, recall that $\Pr(y | x)$ is $K_y$-Lipschitz in the embedding space; that is, $|\Pr(y = 1 | x) - \Pr(y = 1 | x')| \le K_y \rho(f(x), f(x'))$. Since $p_{y, j}$ is $\Pr(y | x)$ averaged over $C_j$, it holds that $|p_{y, x, j} - p_{y, j}| \le K_y d_j$, where $d_j$ is the diameter of $C_j$. We then have that $p_{y, x, j} \le K_y d_j + p_{y, j}$, and since $|(1 - p_{y, x, j}) - (1 - p_{y, j}) | \le K_y d_j$, we also have that $1 - p_{y, x, j} \le 1 - p_{y, j} + K_y d_j$. Therefore, the KL-divergence is bounded by
\begin{align*}
D_{\text{KL}}(\Pr(y | x, C_j) || \Pr(y | C_j)) &\le p_{y, x, j} \log \frac{K_y d_j + p_{y, j}}{p_{y, j}} + (1 - p_{y, x, j}) \log \frac{K_y d_j + (1 - p_{y, j})}{1 - p_{y, j}} \\
&\le p_{y, x, j} \cdot \frac{K_y d_j}{p_{y, j}} + (1 - p_{y, x, j}) \cdot \frac{K_y d_j}{1 - p_{y, j}},
\end{align*}

where we use the fact that $\log(1 + x) \le x$. Plugging this back into~\eqref{eq:y_kl},
\begin{align*}
\E{x}{D_{\text{KL}}(\Pr(y|x) || \Pr(y | C(x)))} &\le \sum_{j = 1}^s \Pr(f(x) \in C_j) \int \Pr(x | C_j) ) \bigg(p_{y, x, j} \cdot \frac{K_y d_j}{p_{y, j}} + (1 - p_{y, x, j}) \cdot \frac{K_y d_j}{1 - p_{y, j}}\bigg) dx \\
&=  \sum_{j = 1}^s \Pr(f(x) \in C_j) \int \Pr(x, y = 1 | C_j) \cdot \frac{K_y d_j}{p_{y, j}} + \Pr(x, y = -1 | C_j) \cdot \frac{K_y d_j}{1 - p_{y, j}} dx \\
&=  \sum_{j = 1}^s \Pr(f(x) \in C_j)\bigg(\Pr( y = 1 | C_j) \cdot \frac{K_y d_j}{p_{y, j}} + \Pr(y = -1 | C_j) \cdot \frac{K_y d_j}{1 - p_{y, j}}\bigg) \\
&= \sum_{j = 1}^s \Pr(f(x) \in C_j) \cdot 2 K_y d_j = 2 K_y d_{\C}.
\end{align*}

Next, we bound $\text{KL}_{C(x)}(\lf_i | y, \lf_i \neq 0)$. Using the same approach, we have that $\text{KL}_{C(x)}(\lf_i | y, \lf_i \neq 0) \le 2 K_{\lf} d_{\C}$. We also have that $\text{KL}_{C(x)}(z_i) \le 2 K_{\lf, 0} d_{\C}$.
\end{proof}

\subsection{Proofs for Section~\ref{subsec:lift}}

\begin{lemma}
When we use $\bm{\lfbar}$ instead of $\bm{\lf}$, the bias term in $L(\bm{\lfbar})$ is at most
\begin{align*}
Bias \le 2d_{\C}K_y + 2m(d_{\C} + 2 \max_i r_i) (K_{\lf} + K_{\lf, 0}).
\end{align*}
\label{lemma:ext_bias}
\end{lemma}

\begin{proof}
The term $\E{x}{D_{\text{KL}}(\Pr(y | x) || \Pr(y | C(x))}$ in the bias is unchanged since the distribution of $y$ given $x$ is not impacted by $\bm{\lf}$. We next look at $\E{x, y, \lfbar_i \neq 0}{D_{\text{KL}}(\Pr(\lfbar_i | y, x, \lfbar_i \neq 0) || \Pr(\lfbar_i | y, C(x), \lfbar_i \neq 0))}$. Using the approach in Lemma~\ref{lemma:kl_biases}, recall that $\Pr(\lfbar_i | y, x, \lfbar_i \neq 0) = \Pr(\lf_i(x) | y, x, \lf_i(x) \neq 0)$  when $\lf_i(x) \neq 0$, and $\Pr(\lf_i(\NN(x)) | y, \NN(x), \lf_i(\NN(x)) \neq 0)$ when $\lf_i(x) = 0$. Therefore, by Assumption~\ref{assumption:lipschitzness}, 
\ifsinglecolumn
$| \Pr(\lfbar_i = 1 | y, \lfbar_i \neq 0, x) - \Pr(\lfbar_i = 1 | y, \lf_i \neq 0, x')| \le K_{\lf} \max\{\\ \rho(f(\NN(x)), f(\NN(x'))), \rho(f(x), f(\NN(x'))), \rho(f(\NN(x)), f(x')), \rho(f(x), f(x'))\}$. 
\else
$| \Pr(\lfbar_i = 1 | y, \lfbar_i \neq 0, x) - \Pr(\lfbar_i = 1 | y, \lf_i \neq 0, x')| \le K_{\lf} \max\{\rho(f(\NN(x)), f(\NN(x'))), \rho(f(x), f(\NN(x'))), \rho(f(\NN(x)), f(x')), \rho(f(x), f(x'))\}$. 
\fi
The greatest possible distance in embedding space between $\NN(x)$ and $\NN(x')$ when $f(x), f(x') \in C_j$ under our method of source extension is $d_j + 2r_i$. We can thus view the extensions as changing the diameter of the subset in Lemma~\ref{lemma:kl_biases}. The rest of the approach remains unchanged, so we get that
\begin{align*}
\E{x, y, \lfbar_i \neq 0}{D_{\text{KL}}(\Pr(\lfbar_i | y, x, \lfbar_i \neq 0) || \Pr(\lfbar_i | y, C(x), \lfbar_i \neq 0))} &\le 2 K_{\lf}(d_{\C} + 2r_i).
\end{align*}

We consider $\E{x}{D_{\text{KL}}(\Pr(\lf_i \neq 0 | x) || \Pr(\lf_i \neq 0 | C(x)))}$. Similarly, $\Pr(\lfbar_i(x) \neq 0 | x)$ is either $\Pr(\lf_i(x) \neq 0 | x)$ or $\Pr(\lf_i(\NN(x)) \neq 0 | \NN(x))$ depending on the region $x$ is in. Therefore, 
\begin{align*}
\E{x}{D_{\text{KL}}(\Pr(\lf_i \neq 0 | x) || \Pr(\lf_i \neq 0 | C(x)))} &\le 2 K_{\lf, 0} (d_{\C} + 2r_i),
\end{align*}

and we obtain the desired bound.
\end{proof}

\extendedacc*

\begin{proof}
We first introduce some notation. Define $S = \{x \in \X: \lf_i(x) \neq 0 \}$ as the support of $\lf_i$, and $\hat{S} = S \cap \D$ as the set of points in $\D$ that $\lf_i$ has coverage on. In particular, $\hat{S}$ consists of points sampled from $\p_{\lf_i}$, and suppose that $|\hat{S}| = n_0$. Define the extended region as $\hat{S}_{r_i} = \{ x \in \X \backslash S: \exists x' \in \hat{S} \; \text{s.t.} \; \rho(f(x), f(x')) \le r_i\}$, and let the distribution of $x$ over this support be $\p_{\hat{S}, r} = \Pr(x | x \in \hat{S}_{r_i})$. With slight abuse of notation, we also use $\p_{\hat{S}, r_i}$ to refer to the joint distribution over $x, y$ with $x$ from $\p_{\hat{S}, r}$. We also use $\hat{S}_{r_i}$ to refer to the support $\hat{S}_{r_i} \times \Y$.

Define the expected error $\varepsilon =  \E{\hat{S} \sim \p_{\lf_i}^{n_0}}{ \Pr_{x, y \sim \p_{\hat{S}, r_i}}(\lfbar_i \neq y| x, \lfbar_i(x) \neq 0)} =  \E{\hat{S} \sim \p_{\lf_i}^{n_0}}{ \Pr_{x, y \sim \p_{\hat{S}, r_i}}(\lfbar_i \neq y| x)}$. Let $\hat{S}$ also be written as a set of $n_0$ random variables $\{x_1, \dots, x_{n_0}\}$. Denote $\NN_{\hat{S}}(x) = \argmin{x' \in \hat{S}}{ \rho(f(x), f(x'))}$ to be $x$'s nearest neighbor in $\hat{S}$ (in the body, this is just referred to as $\NN(x)$), so $\lfbar_i(x) := \lf_i(\NN_{\hat{S}}(x))$ for $x \in \hat{S}_{r_i}$. Then, we decompose $\varepsilon$ based on which point in $\hat{S}$ is $x$'s nearest neighbor:
\begin{align}
\varepsilon &= \Pr_{\substack{\hat{S} \sim \p_{\lf_i}^{n_0}\\ x, y \sim \p_{\hat{S}, r_i}}} (\lf_i(\NN_{\hat{S}}(x)) \neq y | x) = \sum_{j = 1}^{n_0} \Pr_{\substack{\hat{S}\sim \p_{\lf_i}^{n_0}, \\ x \sim \p_{\hat{S}, r_i}}}(\NN_{\hat{S}}(x) = x_j) \cdot \;\; \Pr_{\mathclap{\substack{\hat{S} \sim \p_{\lf_i}^{n_0} \\ x, y, \sim \p_{\hat{S}, r_i}}}} \;\;(\lf_i(x_j) \neq y | \NN_{\hat{S}}(x) = x_j). \label{eq:knn2}
\end{align}

 Let $y_j$ denote the label corresponding to $x_j$, drawn from $\p_{\lf_i}(\cdot | x_j)$. The probability $\Pr_{\hat{S} \sim \p_{\lf_i}^{n_0}, x, y, \sim \p_{\hat{S}, r_i}}(\lf_i(x_j) \neq y | \NN_{\hat{S}}(x) = x_j)$ can be further decomposed into two cases: when $\lf(x_j) = y_j, y_j \neq y$ and when $\lf(x_j) \neq y_j, y_j = y$. That is,
\begin{align}
&\Pr_{\mathclap{\substack{\hat{S} \sim \p_{\lf_i}^{n_0} \\ x, y, \sim \p_{\hat{S}, r_i}}}} \;(\lf_i(x_j) \neq y | \NN_{\hat{S}}(x) = x_j) \label{eq:knn1}\\
= & \;\;\Pr_{\mathclap{\substack{\hat{S} \sim \p_{\lf_i}^{n_0} \\ x, y \sim \p_{\hat{S}, r_i}, \\ y_j \sim \p_{\lf_i}(\cdot | x_j)}}} \;\;(\lf_i(x_j) = y_j, y_j \neq y | \NN_{\hat{S}}(x) = x_j) + \;\;\Pr_{\mathclap{\substack{\hat{S} \sim \p_{\lf_i}^{n_0} \\ x, y \sim \p_{\hat{S}, r_i}, \\ y_j \sim \p_{\lf_i}(\cdot | x_j)}}} \;\;(\lf_i(x_j) \neq y_j, y_j = y | \NN_{\hat{S}}(x) = x_j) \nonumber \\
\le & \;\;\Pr_{\mathclap{\substack{\hat{S} \sim \p_{\lf_i}^{n_0}, \\ y_j \sim \p_{\lf_i}(\cdot | x_j)}}} \;\;(\lf_i(x_j) = y_j, \exists (x, y) \in \hat{S}_{r_i}: \NN_{\hat{S}}(x) = x_j, y_j \neq y) + \;\Pr_{\mathclap{\substack{\hat{S} \sim \p_{\lf_i}^{n_0}, \\ y_j \sim \p_{\lf_i}(\cdot | x_j)}}}\;\;(\lf_i(x_j) \neq y_j, \exists (x, y) \in \hat{S}_{r_i}: \NN_{\hat{S}}(x) = x_j, y_j = y). \nonumber
\end{align}

Next, we recall the definition of $\hat{S}_{r_i}$ and observe that $\NN_{\hat{S}}(x) = x_j$ implies that $\rho(f(x), f(x')) \le r_i$. These allow us to write the probability only over one $(x_j, y_j) \sim \p_{\lf_i}$ rather than $\hat{S}$, and so the expression in~\eqref{eq:knn1} satisfies
\begin{align*}
\Pr_{\mathclap{\substack{\hat{S} \sim \p_{\lf_i}^{n_0} \\ x, y, \sim \p_{\hat{S}, r_i}}}} \;\;(\lf_i(x_j) \neq y | \NN_{\hat{S}}(x) = x_j) \le &\Pr_{x_j, y_j \sim \p_{\lf_i}}(\lf_i(x_j) = y_j, \exists (x, y) \in \X \backslash S: \rho(f(x_j), f(x)) \le r_i, y_j \neq y) \\
+ &\Pr_{x_j, y_j \sim \p_{\lf_i}}(\lf_i(x_j) \neq y_j, \exists (x, y) \in \X \backslash S: \rho(f(x_j), f(x)) \le r_i, y_j = y).
\end{align*}

The first probability on the RHS can be written as $\Pr_{x_j, y_j \sim \p_{\lf_i}}(\lf_i(x_j) = y_j | \exists (x, y) \in \X \backslash S: \rho(f(x_j), f(x)) \le r_i, y_j \neq y) \Pr_{x_j, y_j \sim \p_{\lf_i}}(\exists (x, y) \in \X \backslash S: \rho(f(x_j), f(x)) \le r_i, y_j \neq y) \le \frac{1 + \beta_i}{2} M(r_i)$, and the second one is at most $\Pr_{x_j, y_j \sim \p_{\lf_i}}(\lf_i(x_j) \neq y_j) = \frac{1 - a_i}{2}$. Therefore, putting this back into~\eqref{eq:knn2}, $\varepsilon \le \frac{1 + \beta_i}{2} M(r_i) + \frac{1 - a_i}{2}$. Since $\bar{a}_i(r_i) = 2(1 - \varepsilon) - 1$, we now have our desired bound 
\begin{align*}
\bar{a}_i(r_i) \ge a_i - (1 + \beta_i)M(r_i).
\end{align*}

\end{proof}

\lift*

We aim to lower bound $H(y | \bm{\lf}, x) - H(y | \bm{\lfbar}, x)$ where only $\lf_i$ is extended to be $\lfbar_i$ with threshold radius $r_i$. 
\begin{align}
H(y &| \bm{\lf}, x) - H(y | \bm{\lfbar}, x) = \E{x ,y, \bm{\lf}}{-\log \Pr(y | \bm{\lf}, x)} + \E{x, y, \bm{\lfbar}}{\log \Pr(y | \bm{\lfbar}, x)} \nonumber \\
&= \E{x, y, \lf_{-i}}{ \E{\lfbar_i}{\log \frac{\Pr(\lfbar_i | x, y) \Pr(\lf_{-i} | x, y) \Pr(y | x)}{\Pr(\lfbar_i, \lf_{-i} | x)} \Big| x, y}  - \E{\lf_i}{\log\frac{\Pr(\lf_i | x, y) \Pr(\lf_{-i} | x, y) \Pr(y | x)}{\Pr(\lf_i, \lf_{-i} | x)}  \Big| x, y}}. \label{eq:H_diff1}
\end{align}

$\Pr(\lf_{-i} | x, y)$ and $\Pr(y | x)$ are the same when using $\lfbar_i$ versus $\lf_i$, so~\eqref{eq:H_diff1} becomes
\begin{align}
H(y | \bm{\lf}, x) - H(y | \bm{\lfbar}, x) = \E{x, y, \lf_{-i}}{ \E{\lfbar_i}{\log \frac{\Pr(\lfbar_i | x, y)}{\Pr(\lfbar_i, \lf_{-i} | x)} \Big| x, y}  - \E{\lf_i}{\log\frac{\Pr(\lf_i | x, y) }{\Pr(\lf_i, \lf_{-i} | x)} \Big| x, y} }. \label{eq:H_diff}
\end{align}

When extending $\lf_i$, there are three regions of interest in input space: $\lf_i(x), \lfbar_i(x) \neq 0$; $\lf_i(x) = 0, \lfbar_i(x) \neq 0$; and $\lf_i(x) = \lfbar_i(x) = 0$. In the first region, $\lfbar_i$ has the exact same behavior as $\lf_i$ since $\lf_i$ has coverage over this region. Therefore, conditioning on $\lf_i(x) \neq 0$, the expectation on the RHS of~\eqref{eq:H_diff} is equal to $0$. Similarly, in the third region where $\lf_i(x) = \lfbar_i(x) = 0$, the extended and original labeler vote exactly the same, so the expectation on the RHS of~\eqref{eq:H_diff} is again equal to $0$. The primary region of interest are the points that previously had no signal from $\lf_i$ but now have signal from $\lfbar_i$. Then,~\eqref{eq:H_diff} becomes
\begin{align}
H(y | \bm{\lf}, x) - H(y | \bm{\lfbar}, x) = p_i \E{y, \lf_{-i}, \lfbar_i(x) \neq 0, \lf_i(x) = 0}{\log \frac{\Pr(\lfbar_i | x, y)}{\Pr(\lfbar_i, \lf_{-i} | x)}   - \log\frac{\Pr(\lf_i = 0| x, y) }{\Pr(\lf_i = 0, \lf_{-i} | x)}}. \label{eq:H_diff2}
\end{align}

We can write $\frac{\Pr(\lf_i = 0| x, y) }{\Pr(\lf_i = 0, \lf_{-i} | x)} = \frac{\Pr(\lf_i = 0 | x)}{\Pr(\lf_i = 0 | x) \Pr(\lf_{-i} | x)} = \frac{1}{\Pr(\lf_{-i} | x)}$ by decomposing the denominator conditional on $y$ and using Lemma~\ref{lemma:abstain}. Using the chain rule on $\Pr(\lfbar_i, \lf_{-i}| x) = \Pr(\lfbar_i | \lf_{-i}, x) \Pr(\lf_{-i} | x)$,~\eqref{eq:H_diff2} is now
\begin{align*}
H(y | \bm{\lf}, x) - H(y | \bm{\lfbar}, x) =p_i \E{y, \lf_{-i}, \lfbar_i(x) \neq 0, \lf_i(x) = 0}{\log \frac{\Pr(\lfbar_i | x, y)}{\Pr(\lfbar_i| \lf_{-i},  x)}}.
\end{align*}

To analyze this expectation, we first look at the case where $y = 1$. Then,
\begin{align}
&\E{y = 1, \lf_{-i}, \lfbar_i(x) \neq 0, \lf_i(x) = 0}{\log \frac{\Pr(\lfbar_i | x, y)}{\Pr(\lfbar_i| \lf_{-i},  x)}} = \label{eq:H_diffy1}\\
\mathbb{E}\bigg[&\Pr(\lfbar_i = 1 | y = 1, x, \lfbar_i \neq 0) \log \frac{\Pr(\lfbar_i = 1 | x, y = 1)}{\Pr(\lfbar_i = 1 | \lf_{-i}, x)} + \Pr(\lfbar_i = -1 | y = 1, x, \lfbar_i \neq 0) \log \frac{\Pr(\lfbar_i = -1 | x, y = 1)}{\Pr(\lfbar_i = -1 | \lf_{-i}, x)} \bigg].\nonumber  
\end{align}

Denote $\alpha_i(x) = \Pr(\lfbar_i = 1 | y = 1, x, \lfbar_i \neq 0)$ as the probability corresponding to $\lfbar_i$'s accuracy parameter. In addition, note that we can write 
\begin{align*}
\Pr(\lfbar_i = 1 | \lf_{-i}, x) &= \Pr(\lfbar_i = 1 | \lf_{-i}, x, y = 1) \Pr(y = 1 | \lf_{-i}, x) + \Pr(\lfbar_i = 1 | \lf_{-i}, x, y = -1) \Pr(y = -1 | \lf_{-i}, x) \\
&= \alpha_i(x) p(x, \lf_{-i}) + (1 - \alpha_i(x))(1 - p(x, \lf_{-i})),
\end{align*}

where $p(x, \lf_{-i})$ is shorthand for $\Pr(y = 1 | \lf_{-i}, x)$ (importantly, it does not depend on $\lfbar_i$) and likewise for $\Pr(\lfbar_i = -1 | \lf_{-i}, x) = \alpha_i(x) (1 - p(x, \lf_{-i})) + (1 - \alpha_i(x)) p(x, \lf_{-i})$. Our expression from~\eqref{eq:H_diffy1} is now
\begin{align}
\mathbb{E}_{y = 1, \lf_{-i}, \lfbar_i(x) \neq 0, \lf_i(x) = 0}\bigg[ &\alpha_i(x) \log \frac{\alpha_i(x)}{\alpha_i(x) p(x, \lf_{-i}) + (1 - \alpha_i(x))(1 - p(x, \lf_{-i}))} \label{eq:H_diffy1_2} \\
&+ (1 - \alpha_i(x)) \log \frac{1 - \alpha_i(x)}{\alpha_i(x) (1 - p(x, \lf_{-i})) + (1 - \alpha_i(x)) p(x, \lf_{-i})}\bigg]. \nonumber 
\end{align}

Note that the expression inside the expectation is convex in both $\alpha_i(x)$ and $p(x, \lf_{-i})$.
\ifsinglecolumn
Define 
\begin{align*}
\alpha_{i, 1} &= \E{y = 1, \lfbar_i(x) \neq 0, \lf_i(x) = 0}{\alpha_i(x)} \\
p_{\lf_{-i}, 1} &= \E{y' = 1, \lf_{-i}, \lfbar_i(x) \neq 0, \lf_i(x) = 0}{\Pr(y = y' | x, \lf_{-i})}.
\end{align*}

$\alpha_{i, 1}$ is the expected accuracy probability over the extended region when $y = 1$, and $p_{\lf_{-i}, 1}$ is the expected label model performance using just $\lf_{-i}$ over the extended region when $y = 1$.
\else
Define $\alpha_{i, 1} = \E{y = 1, \lfbar_i(x) \neq 0, \lf_i(x) = 0}{\alpha_i(x)}$ to be the expected accuracy probability over the extended region when $y = 1$, and $p_{\lf_{-i}, 1} = \E{y' = 1, \lf_{-i}, \lfbar_i(x) \neq 0, \lf_i(x) = 0}{\Pr(y = y' | x, \lf_{-i})}$ to be the expected label model performance using just $\lf_{-i}$ over the extended region when $y = 1$.
\fi
 Then, this expression from~\eqref{eq:H_diffy1_2} is at least
\begin{align}
&\alpha_{i, 1} \log \frac{\alpha_{i, 1}}{\alpha_{i, 1} p_{\lf_{-i}, 1} + (1 - \alpha_{i, 1})(1 - p_{\lf_{-i}, 1})} \label{eq:H_diffy1_3}\\
&+ (1 - \alpha_{i, 1}) \log \frac{1 - \alpha_{i, 1}}{\alpha_{i, 1} (1 - p_{\lf_{-i}, 1}) + (1 - \alpha_{i, 1}) p_{\lf_{-i}, 1}}. \nonumber 
\end{align}

We look at the case where $y = -1$. Similarly, we get
\begin{align}
\mathbb{E}_{y = -1, \lf_{-i}, \lfbar_i(x) \neq 0, \lf_i(x) = 0}\bigg[&\alpha_i(x) \log \frac{\alpha_i(x)}{\alpha_i(x)(1 - p(x, \lf_{-i})) + (1 - \alpha_i(x))p(x, \lf_{-i})} \label{eq:H_diffy-1} \\
&+ (1 - \alpha_i(x)) \log \frac{1 - \alpha_i(x)}{\alpha_i(x) p(x, \lf_{-i}) + (1 - \alpha_i(x))(1 - p(x, \lf_{-i}))}\bigg]. \nonumber
\end{align}

Again, define $\alpha_{i, -1} = \E{y = -1, \lfbar_i(x) \neq 0, \lf_i(x) = 0}{\alpha_i(x)}$ and $p_{\lf_{-i}, -1} = \E{y' = -1, \lf_{-i}, \lfbar_i(x) \neq 0, \lf_i(x) = 0}{\Pr(y = y' | x, \lf_{-i})}$, and by Jensen's inequality we have that~\eqref{eq:H_diffy-1} is at least
\begin{align}
&\alpha_{i, -1} \log \frac{\alpha_{i, -1}}{\alpha_{i, -1}p_{\lf_{-i}, -1} + (1 - \alpha_{i, -1})(1 - p_{\lf_{-i}, -1})} \label{eq:H_diffy-1_2}\\
&+ (1 - \alpha_{i, -1}) \log \frac{1 - \alpha_{i, -1}}{\alpha_{i, -1} (1 - p_{\lf_{-i}, -1}) + (1 - \alpha_{i, -1})p_{\lf_{-i}, -1}}. \nonumber 
\end{align}

Therefore, $\E{y, \lf_{-i}, \lfbar_i(x) \neq 0, \lf_i(x) = 0}{\log \frac{\Pr(\lfbar_i | x, y)}{\Pr(\lfbar_i | \lf_{-i}, x)}}$ is lower bounded by the weighted sum of $\Pr(y = 1 | \lf_{-i}, \lfbar_i(x) \neq 0, \lf_i(x) = 0)$ times~\eqref{eq:H_diffy1_3} and $\Pr(y = -1 | \lf_{-i}, \lfbar_i(x) \neq 0, \lf_i(x) = 0)$ times~\eqref{eq:H_diffy-1_2}.
Since~\eqref{eq:H_diffy1_3} and~\eqref{eq:H_diffy-1_2} are convex in $\alpha_{i, 1}, p_{\lf_{-i}, 1}$ and $\alpha_{i, -1}, p_{\lf_{-i}, 1}$ respectively, we can define $\alpha_i = \Pr(y = 1 | \lf_{-i}, \lfbar_i(x) \neq 0, \lf_i(x) = 0) \cdot \alpha_{i, 1} +  \Pr(y = -1 | \lf_{-i}, \lfbar_i(x) \neq 0, \lf_i(x) = 0) \cdot \alpha_{i, -1} =  \E{\lfbar_i(x) \neq 0, \lf_i(x) = 0}{\alpha_i(x)}$ as a notion of $\lfbar_i$'s accuracy in the region where we extend $\lf_i$. We also define $p_{\lf_{-i}} = \Pr(y = 1 | \lf_{-i}, \lfbar_i(x) \neq 0, \lf_i(x) = 0) \cdot p_{\lf_{-i}, 1} + \Pr(y = -1 | \lf_{-i}, \lfbar_i(x) \neq 0, \lf_i(x) = 0) \cdot p_{\lf_{-i}, -1} = \E{y', \lf_{-i}, \lfbar_i(x) \neq 0, \lf_i(x) = 0}{\Pr(y = y' | \lf_{-i}, x}$ as the label model's probability of outputting the correct label in our region of interest when relying on only $\lf_{-i}$. Then, we have that
\begin{align*}
H(y | \bm{\lf}, x) - H(y | \bm{\lfbar}, x) \ge p_i \bigg(&\alpha_i \log \frac{\alpha_i}{\alpha_i p_{\lf_{-i}} + (1 - \alpha_i)(1 - p_{\lf_{-i}})} \\
&+ (1 - \alpha_i) \log \frac{1 - \alpha_i}{(1 - \alpha_i) p_{\lf_{-i}} + \alpha_i (1 - p_{\lf_{-i}})} \bigg).
\end{align*}

We can lower bound the expression in the parentheses. Define $g(x) = x \log \frac{x}{xp + (1 - x)p} + (1 - x) \log \frac{1 - x}{(1 - x)p + x(1 - p)}$ for some constant $p$. We claim that $g(x) \ge h(x) = 8 (1 - p)^2 (x - 0.5)^2$ for $x \in [0, 1]$. Note that $g(0.5) = h(0.5) = 0$. To show that $g(x) \ge h(x)$, it suffices to show that $g'(x) > h'(x)$ for $x > 0.5$, and $g'(x) < h'(x)$ for $x < 0.5$. $g'(x) = \frac{1 - p}{xp + (1 - x)(1 - p)} + \frac{p - 1}{x(1 - p) + (1 - x)p} + \log \frac{x}{xp + (1 - x)(1 -p)} - \log \frac{1 - x}{(1-x)p + x(1- p)} $, and $h'(x) = 16(1 - p)^2 (x - 0.5)$. Again, note that $g'(0.5) = h'(0.5) = 0$, so we want to show that $g''(x) > h''(x)$ for all $x \in [0, 1]$.  $g''(x) = -\frac{(1 - p)(2p - 1)}{(xp + (1 - x)(1- p ))^2} - \frac{(p - 1)(1 - 2p)}{(x(1 - p) + (1 - x)p)^2} + \frac{1 - p}{x(xp + (1 - x)(1 - p))} + \frac{1 - p}{(1 - x)(x(1 - p) + (1 - x)p)}$, and $h''(x) = 16(1 - p)^2$. It is easy to check that $g''(x)$ obtains a minimum at $x = 0.5$. We compute that $g''(0.5) = 16(1 - p)^2$, which demonstrates that $g(x) \ge 8(1 - p)^2 (x - 0.5)^2$. We thus get
\begin{align*}
H(y | \bm{\lf}, x) - H(y | \bm{\lfbar}, x) \ge 8p_i(1 - p_{\lf_{-i}})^2 \Big( \alpha_i - \frac{1}{2} \Big)^2.
\end{align*}

We know that $\alpha_i = \frac{1 + \bar{a}_i(r_i)}{2}$, so our final bound is
\begin{align*}
H(y | \bm{\lf}, x) - H(y | \bm{\lfbar}, x) \ge 8p_i (1 - p_{\lf_{-i}})^2 \cdot  \frac{\bar{a}_i(r_i)^2}{4} = 2p_i (1 - p_{\lf_{-i}})^2 \cdot \bar{a}_i(r_i)^2.
\end{align*}


\section{Experimental Details}
\label{sec:supp_details}

We describe additional details about each task, 
including details about data sources (Section~\ref{sec:supp_details_dataset}),
supervision sources (Section~\ref{sec:supp_details_lfs}),
and setting extension thresholds (Section~\ref{sec:supp_details_thresholds}).

\subsection{Dataset Details}
\label{sec:supp_details_dataset}


\begin{table}[ht!]
    \centering
    \begin{tabular}{lcccccc}
        \toprule
        \textbf{Task} (Embedding) & $T$ & $m/T$ & \textbf{Prop} & $N_{train}$ & $N_{dev}$ & $N_{test}$   \\ \midrule
        \spam\                    & 1   & 10    & 0.49          & 1,586       & 120       & 250          \\
        \weather\                 & 1   & 103   & 0.53          & 187         & 50        & 50           \\
        \spouse\                  & 1   & 9     & 0.07          & 22,254      & 2,811     & 2,701        \\
        \basketball\              & 8   & 4     & 0.12          & 3,594       & 212       & 244          \\ 
        \commercial\              & 3   & 4     & 0.32          & 64,130      & 9,479     & 7,496        \\
        \tennis\                  & 9  & 6     & 0.34          & 6,959       & 746       & 1,098        \\
        \bottomrule
    \end{tabular}
    \caption{
    Details for each dataset.
    $T$: the number of related elements modeled by the weak supervision label
    model.
    $m/T$: the number of supervision sources per element.
    \textbf{Prop}: The proportion of positive examples in each dataset.
    $N_{train}$: The size of the unlabeled training set.
    $N_{dev}$: The size of the labeled dev set.
    $N_{test}$: The size of the held-out test set.
    }
    \label{table:stats}
\end{table}

Table~\ref{table:stats} provides details on train/dev/test splits for each
dataset, as well as statistics about the positive class proportion and the
number of labeling functions.
Additional details about each dataset are provided below.

\paragraph{\spam}
We use the dataset as provided by
Snorkel\footnote{https://www.snorkel.org/use-cases/01-spam-tutorial} and those
train/dev/test splits.

\paragraph{\weather, \spouse}
These datasets are used in~\cite{Ratner18} and~\cite{fu2020fast} for
evaluation, and we use the train/dev/test splits from
those works (\weather\ is called \textbf{Crowd} in that work).

\paragraph{\basketball}
This dataset is a subset of ActivityNet and was used for evaluation
in~\cite{sala2019multiresws} and~\cite{fu2020fast}.
We use the train/dev/test splits from those works.

\paragraph{\commercial}
We use the dataset from~\cite{fu2019rekall, hong2021analysis} and~\cite{fu2020fast} and the
train/dev/test splits from those works.

\paragraph{\tennis}
We use the dataset from~\cite{fu2020fast} and the train/dev/test splits from
those works.

\subsection{Supervision Sources}
\label{sec:supp_details_lfs}

Supervision sources are expressed as short Python functions.
Each source relied on different information to assign noisy labels:

\paragraph{\spam, \weather, \spouse}
For these tasks, we used the same supervision sources as used in previous
work~\citep{Ratner18, fu2020fast}.
These are all text classification tasks, so they rely on text-based heuristics
such as the presence or absence of certain words, or particular regex patterns.

\paragraph{\basketball, \commercial, \tennis}
Again, we use sources from previous work~\citep{sala2019multiresws, fu2020fast}.
For \basketball, these sources rely on an off-the-shelf object detector to
detect balls or people, and use heuristics based on the average pixel of the
detected ball or distance between the ball and person to determine whether the
sport being played is basketball or not.
For \commercial, there is a strong signal for the presence or absence of
commercials in pixel histograms and the text; in particular, commercials are
book-ended on either side by sequences of black frames, and commercial segments
tend to have mixed-case or missing transcripts (whereas news segments are in
all caps).
For \tennis, we use an off-the-shelf pose detector to provide primitives for the
weak supervision sources.
The supervision sources are heuristics based on the number of people on court
and their positions.
Additional supervision sources use color histograms of the frames (i.e., 
how green the frame is, or whether there are enough white pixels for the court
markings to be shown).

\subsection{Setting $r_i$ and $s$}
\label{sec:supp_details_thresholds}

We tune $r_i$ using the dev set in two steps.
First, we set all the $r_i$ to the same value $r$ and use grid search over $r$.
Then, we perform a series of small per-coordinate searches for a subset of the
labeling functions to optimize individual $r_i$ values.
For labeling functions with full coverage, we set the threshold to have no
extensions.

Tuning $s$ is done independently from $r_i$. Once we have the best performing $r_i$ values, we search for the best possible $s$ from one to ten. We obtain the partition by performing K-means clustering with $K = s$.

Now we report thresholds in terms of \textit{cosine similarities} (note that
this is a different presentation than in terms of distances).
For \spam, all thresholds are set to $0.844$, except for weak sources $1$, $2$, and $7$, which
have thresholds $0.864$, $0.854$ and $0.804$ respectively. The best $s$ is 2.
For \weather, all thresholds are set to $0.2$, and the best $s$ is 3.
For \spouse, all thresholds are set to $0.9275$, except for weak sources $2$ and $3$, which have thresholds $0.8385$ and $0.9$. The best $s$ is 8.
For \basketball, thresholds are set to $[0.42, 0.97, 0.52, 0.42]$ and $s$ is set to 2.
For \commercial, thresholds are set to $ [.6, .35, .35, .65]$ and $s$ is set to $3$.
For \tennis, thresholds are set to $[0.11, 0.11 0.11, 0.85, 0.11, 0.11]$ and $s$ is set to $2$.

In our experimentss, class balance $\Pr(y | C_j)$ is estimated from the dev set.

\subsection{Adapters}
We describe adapter experimental details in the main results. For each dataset, we train single-layer adapters with gradient descent. Because this requires training labels, we consider two training setups: (1) splitting the validation set into a new 80\% training set and 20\% held-out validation set, and (2) using weak-supervision methods (WS-LM) combined with labeling functions to generate pseudolabels for the training data. 

For both, we train adapters using the OpenAI GPT-3 Ada embeddings for NLP tasks and OpenAI CLIP embeddings for video tasks.. We train with 50 epochs and early stopping, and sweep over the following hyperparameters: learning rate $\in \{1e-3, 1e-2, 1e-1\}$,
weight decay $\in \{5e-4, 0\}$,
momentum $\in \{0, 0.9\}$.

The best performing model (based on held-out validation set accuracy for Spam and Weather datasets, held-out validation F1-score for all other datasets), was then evaluated on the test set. 

For the linear models, the best hyperparameters are as follows: for \spam, we use $1e-1$ learning rate, $5e-4$ weight decay, and $0.9$ momentum. For \weather, we use $1e-1$ learning rate, $5e-4$ weight decay, and $0.9$ momentum. For \spouse, we use $1e-2$ learning rate, $5e-4$ weight decay, and $0$ momentum. For \basketball, we use $1e-3$ learning rate, $0$ weight decay, and $0.9$ momentum. For \commercial, we use $1e-1$ learning rate, $5e-4$ weight decay, and $0$ momentum. For \tennis, we use $1e-1$ learning rate, $5e-4$ weight decay, and $0$ momentum.

For the MLPs, the best hyperparameters are as follows: for \spam, we use $0.1$ learning rate, $0$ weight decay, $0.9$ momentum, and $512$ hidden layer dimension. For \weather, we use $1e-1$ learning rate, $5e-4$ weight decay, $0.9$ momentum, and $256$ hidden layer dimension. For \spouse, we use $1e-3$ learning rate, $0$ weight decay, $0$ momentum, and $256$ hidden layer dimension. For \basketball, we use $1e-2$ learning rate, $0$ weight decay, $0.9$ momentum, and $512$ hidden layer dimension. For \commercial, we use $1e-1$ learning rate, $5e-4$ weight decay, $0.9$ momentum, and $512$ hidden layer dimension. For \tennis, we use $1e-2$ learning rate, $5e-4$ weight decay, $0.9$ momentum, and $256$ hidden layer dimension.

\paragraph{\sysname-Adapter}
In addition to evaluating \sysname\ on its own against linear adapters, we also demonstrate further boosts when combining the \sysname\ predictions with Adapters. For this approach, we first create training sets by combining the 80\% split of the original validation set and the original training set. To get labels, we use the ground-truth labels for the former, and the \sysname\ predictions on the training set for the latter. To get data inputs, we tune between using the same data embeddings as in the original datasets, and optionally concatenating the \sysname\ predictions as an additional input dimension to the embeddings. In the setup, for validation and test sets, we also concatenate the \sysname\ predictions to the embeddings. For the \spouse\ dataset, we do this concatenation, as we found it to improve the validation set F1-score. For all others, we use the original embeddings.
When the weak labels are not very accurate ($<75\%$ accuracy on dev), we downsample the train points (otherwise they would degrade performance from ground-truth dev labels).
This allows performance on \basketball\ to be strong even though \sysname\ accuracy is relatively low.

We tune hyperparameters in the same way as the other adapters. The best hyperparameters are as follows: for \spam, we use $0.1$ learning rate, $0$ weight decay, $0.9$ momentum. For \weather, we use $0.1$ learning rate, $5e-4$ weight decay, $0.9$ momentum. For \spouse, we use $1e-3$ learning rate, $5e-4$ weight decay, $0.9$ momentum. For \basketball, we use $10.1$ learning rate, $5e-4$ weight decay, $0$ momentum. For \commercial, we use $0.1$ learning rate, $0$ weight decay, $0.9$. For \tennis, we use $1e-1$ learning rate, $5e-4$ weight decay, $0$ momentum.

\section{Additional Experimental Results}\label{sec:supp_exp}

\subsection{MLP Adapters}

\begin{table}[h!]
    \centering
    \begin{tabular}{lcc}
    \toprule                                                
    \textbf{Task}  & \textbf{\sysname\ ($s$)}  \\
    \midrule                                                 
    \spam          & \textbf{96.8} (2)        \\
    \weather       & \textbf{95.3} (3)        \\
    \spouse        & \textbf{17.0} (6)       \\
    \midrule                                                   
    \basketball    & \textbf{81.7} (2)         \\
    \commercial    & \textbf{93.4} (3)      \\
    \tennis        & \textbf{83.4} (1)       \\
    \bottomrule
    \end{tabular}

    \caption{MLP Adapter performance. Scores are in F1, except for Spam and Weather (accuracy).}
    \label{table:mlp_adapter}
\end{table}

We also evaluated adapters using 3-layer MLPs as alternatives to the linear adapters. We considered MLPs with 512 or 256 dimensional hidden-layers with the ReLU nonlinear activation function. We report the results in Table \ref{table:mlp_adapter}.
Performance is similar to the linear adapters, but the MLP adapters are slightly more expensive to train.
We focus on a simple linear probe for \sysname-Adapter and the main experiments for simplicity.

\subsection{Additional Measures of Smoothness}\label{sec:supp_smooth}


\begin{figure*}[t]
    \centering
    \begin{minipage}{2in}
        \begin{flushright}
            \small
            \begin{tabular}{lr}    
                \toprule                                                       
                \textbf{Embedding} & F1-score \\
                \midrule         
                Raw pixel          & 19.3  \\
                RN-101        & 31.1\\
                BiT-M          & 42.5  \\
                \textbf{CLIP}            & \textbf{69.6}  \\
                \bottomrule
            \end{tabular}
    \end{flushright}
    \end{minipage}
    \begin{minipage}{4.5in}
        \centering
        \includegraphics[width=4.5in]{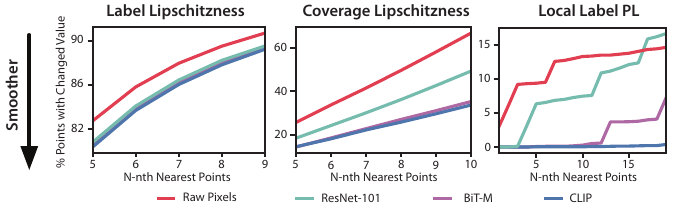}
    \end{minipage}
    \begin{minipage}{2in}
        \begin{flushright}
            \small
            \begin{tabular}{lr}    
                \toprule                                                     
                \textbf{Prompting} & F1-score \\
                \midrule
                No Prompt          & 48.5  \\
                Prompt at Beginning          & 50.2  \\
                \textbf{Prompt at End}          & \textbf{52.2} \\
                \bottomrule
            \end{tabular}
        \end{flushright}
    \end{minipage}
    \begin{minipage}{4.5in}
        \centering
        \includegraphics[width=4.5in]{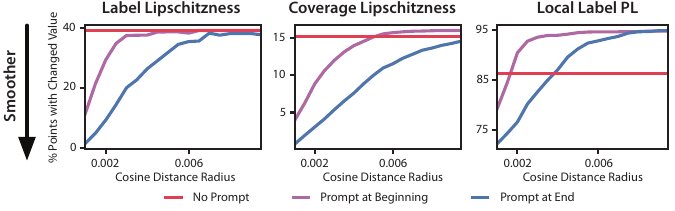}
    \end{minipage}
    \caption{
        Top: \sysname\ performance and smoothness measurements of CLIP, BiT-M, ResNet-101, and raw pixels as embeddings for \basketball.
        Bottom: \sysname\ performance and smoothness measurements of no prompting, prompting at beginning, and prompting at end in GPT-3 for \spouse.
    }
    \label{smoothness_all}
\end{figure*}

Figure~\ref{smoothness_all} reports two additional measurements of smoothness on \basketball\ and \spouse---coverage Lipschitzness and local label probabilistic Lipschitzness (see Section~\ref{sec:theory} for the formal definitions). Trends match label Lipschitzness.

To measure label Lipschitzness, the property that $|\Pr(y = 1 | x) - \Pr(y = 1 | x')| \le K_y \rho(f(x), f(x'))$, we observe that
\begin{align*}
|\Pr(y = 1 | x) - \Pr(y = 1 | x')| &= | \E{}{\ind{y = 1 | x} - \E{}{\ind{y = 1 | x'}}} \\
&\le \E{}{|\ind{y = 1 | x} - \ind{y = 1 | x'} |} \\
&= \E{}{\ind{y \neq y'}} = \Pr(y \neq y')
\end{align*}

by Jensen's inequality. Therefore, we estimate $\Pr(y = y')$ on data as an upper bound on label Lipschitzness. We do this by computing the average percentage of points in some local region (defined either by a radius or by nearest neighbors) around a given point where the label is different from that of the given point.

For source Lipschitzness, the sources in practice are unimodal and hence $K_{\lf} = 0$. 

For coverage Lipschitzness, we note that $|\Pr(\lf_i \neq 0 | x) - \Pr(\lf_i \neq 0 | x')| \le \Pr(\ind{\lf_i(x) \neq 0} != \ind{\lf_i(x) \neq 0})$, so we estimate this probability on data as an upper bound. This is done by computing the average percentage of points that abstain in some local region around a point that has coverage, and vice versa. We average over all sources.

Finally, for local label probabilistic Lipschitzness, we follow Definition~\ref{def:pl}. For each point in the support of $\lf_i$, we search if there exists a nearby point within radius $r$ (or $k$-th nearest neighbor) such that this nearby point is not in the support and has a label differerent from that of the given point. We compute the percentage of points in the support that satisfy this property. We average over all sources.
 
To read $K_y, K_{\lf, 0}, M$ from Figure~\ref{smoothness_all}, they can each be viewed as the slope of the linear function that upper bounds the smoothness curve. Note that for the curves that appear flat (i.e. no prompt), these constants are very large, as there is an initial sharp increase in the percentage of points with changed value.

\subsection{Synthetic Experiments} \label{sec:supp_exp_synthetics}

We evaluate \sysname\ on synthetic data to confirm our insights about 1) how generalization error for $\hat{\Pr}(y | \bm{\lf}, x)$ demonstrates a bias-variance tradeoff depending on the number of partitions, and 2) how additional lift depends on setting the threshold radius based on the original weak source's accuracy and the embedding's probabilistic Lipschitzness. 

\begin{figure}[t]
  \centering
  \includegraphics[width=3in]{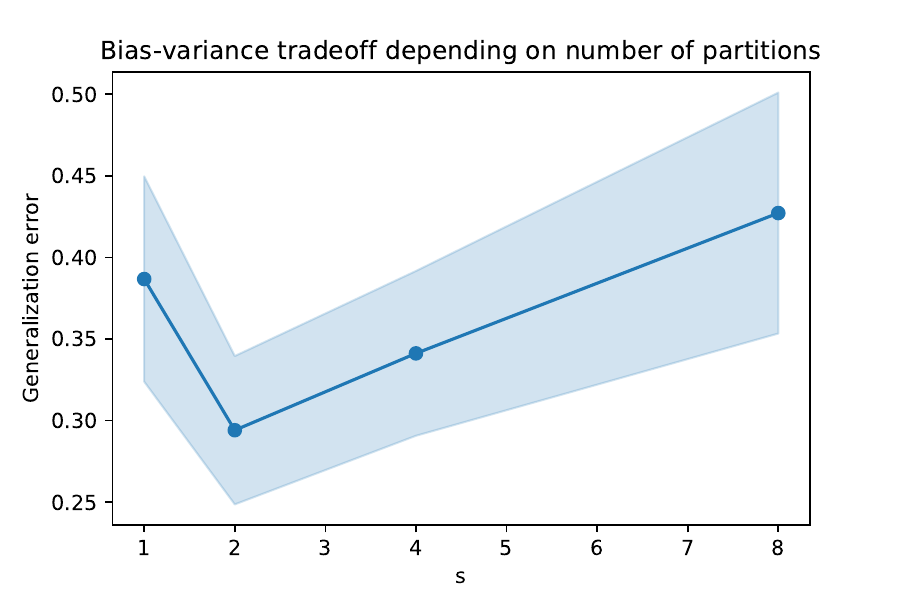}
  \caption{
    The bias-variance tradeoff in the generalization error based on $s$, the number of partitions used in our approach. When too few partitions are used, the accuracy estimates are not fine-grained and do not sufficiently approximate the true conditional distribution $\Pr(y, \bm{\lf} | x)$, resulting in large bias. When too many partitions are used, the variance increases due to sampling error on individual partitions. 
  }
  \label{fig:partition_synthetic}
\end{figure}

First, we conduct a synthetic experiment to understand how the number of partitions $s$ controls the bias-variance tradeoff in generalization error of $\hat{\Pr}(y | \bm{\lf}, x)$ (Theorem~\ref{thm:gen_err}. 
We generate two sets of canonical parameters and use them in~\eqref{eq:pgm} to generate $(y, \bm{\lf})$ from two different distributions, $\p_1$ and $\p_2$ over an embedding space. We generate $1000$ points each for $\p_1$ and $\p_2$ to form datasets $\D_1$ and $\D_2$, which are then concatenated to form a dataset $\D$ of $2000$ points. 
We first run Algorithm~\ref{alg:main} with $s = 1$, which means that we estimate only one set of parameters over $\D$ despite the dataset consisting of two different conditional distributions.
We then set $s = 2$ and estimate the parameters of $\p_1$ and $\p_2$ separately over $1000$ points each.
Finally, we set $s = 4$ and $s = 8$ by dividing each of $\D_1$ and $\D_2$ into $2$ subsets of $500$ points and $4$ subsets of $250$ points, respectively.
For each of these, we compute the average cross-entropy loss (over $s$) of our label model. 
Figure~\ref{fig:partition_synthetic} plots how the generalization error changes with the number of partitions $s$. We plot the mean and $95\%$ confidence interval over ten random initializations of canonical parameters and datasets drawn according to them.
It demonstrates a bias-variance tradeoff: when $s = 1$, we estimate one set of parameters over the entire dataset rather than the two true sets of parameters, and this approach hence does not capture the distinctions in input space among the source accuracies. 
As a result, a low $s$ results in high bias, contributing to large generalization error.
On the other hand, when $s = 4$ or $8$, our approach is correctly estimating $\D_1$ and $\D_2$ separately but is using much less data to do so. 
This approach has higher sampling error, which worsens variance and contributes to large generalization error.

\begin{figure}[t]
  \centering
  \includegraphics[width=4in]{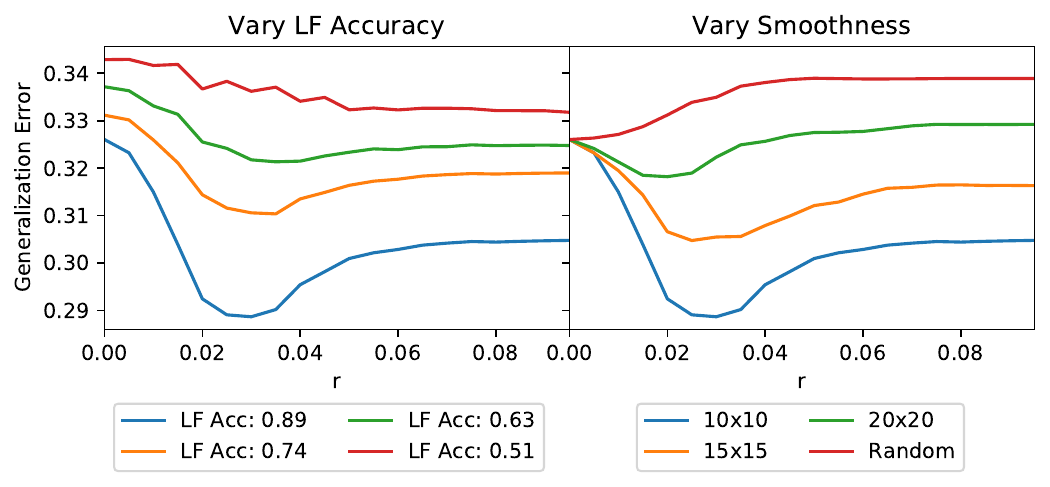}
  \caption{
    Reduction in generalization error from extending labeling functions of
    varying accuracies (left), and on
    embedding spaces of varying smoothness (right). LF refers to a weak source's labeling 	    function.
  }
  \label{fig:extension_synthetic}
\end{figure}

Next, we conduct a synthetic experiment to understand how setting the threshold radius of the extended weak source controls improvement in generalization error as a function of the original average source accuracy and the probabilistic Lipschitzness of the FM embedding (Theorem~\ref{thm:lift}). 
Suppose for simplicity that $s = 1$ and that $\Pr(y, \bm{\lf})$ is modeled the same way as $\Pr(y, \bm{\lf} | x)$ in~\eqref{eq:pgm}. This assumption reduces to previous weak supervision settings but allows us to isolate the effect of extending a source.
We create an embedding space over $10000$ uniformly sampled points in $[0, 1]^2$ with a fixed class balance $\Pr(y)$ and $m = 3$ labeling functions, where only $\lf_1$ is extended. 
To understand the impact of a labeling function's accuracy, we fix a task distribution by assigning $Y$ labels in a $10 \times 10$ ``checkerboard'' pattern and run our algorithm on four versions of $\lf_1$ with varying average accuracies, keeping $\lf_1$'s support consistent. 
In Figure \ref{fig:extension_synthetic} (left), we extend $\lf_1$ based on $r$ for each of the four versions of the labeling function.
This confirms that extending a highly accurate labeling function results in greater generalization lift. 
To understand the impact of Lipschitzness of the task distribution, we produce four distributions of $Y$ over the embedding space, three of which follow a checkerboard clustering pattern (such that more divisions mean less smoothness), and one that spatially distributes the values of $Y$ at random. 
For both experiments, we run our approach with threshold radius varying from $0$ to $0.1$ in increments of $0.005$.
In Figure \ref{fig:extension_synthetic} (right), each curve represents performance of the same high average accuracy labeling function $(a_1 = 0.89$) over embeddings of varying Lipschitzness. 
This confirms that the greatest improvement due to an extension occurs for the smoothest embedding. 
Lastly, both of these graphs illustrate the tradeoff in setting a threshold radius, confirming the theoretical insight that this quantity must be chosen carefully to ensure lift from using $\lfbar_1$ over $\lf_1$.

\end{document}